\documentclass{article}

\usepackage[preprint]{neurips_2025}

\usepackage{algorithm}
\usepackage{algorithmic}
\usepackage{multicol}
\usepackage{wrapfig}
\usepackage{xcolor}
\usepackage{doi}
\usepackage{hyperref}
\hypersetup{colorlinks=true,linkcolor=blue,urlcolor=blue,citecolor=blue}
\usepackage{placeins} 

\usepackage{microtype}
\usepackage{graphicx}
\usepackage{subfigure}
\usepackage{booktabs} 
\usepackage[inline]{enumitem}
\usepackage{adjustbox} 

\makeatletter
\def\ttl@useclass#1#2{%
  \@ifstar
    {\ttl@labeltrue\@dblarg{#1{#2}}}
    {\ttl@labeltrue\@dblarg{#1{#2}}}}
\makeatother

\usepackage{mathtools}

\DeclarePairedDelimiter{\inner}{\langle}{\rangle}
\newcommand{\Wass}{\mathcal{W}}
\newcommand{\calX}{\mathcal{X}}
\newcommand{\calY}{\mathcal{Y}}
\newcommand{\code}[1]{\texttt{#1}}

\DeclareMathOperator{\diag}{diag}
\DeclareMathOperator{\supp}{supp}
\DeclareMathOperator*{\argmin}{arg\,min}
\DeclareMathOperator*{\argmax}{arg\,max}
\newcommand{\vmu}{{\bm{\mu}}}
\newcommand{\vnu}{{\bm{\nu}}}
\newcommand{\va}{{\bm{a}}}
\newcommand{\vb}{{\bm{b}}}
\newcommand{\vf}{{\bm{f}}}
\newcommand{\vg}{{\bm{g}}}
\newcommand{\vw}{{\bm{w}}}
\newcommand{\N}{{\mathbb{N}}}
\newcommand{\R}{{\mathbb{R}}}
\newcommand{\sX}{{\mathbb{X}}}
\newcommand{\sY}{{\mathbb{Y}}}
\newcommand{\tP}{{\mathbf{P}}}
\newcommand{\tK}{{\mathbf{K}}}
\newcommand{\ones}{{\mathbf{1}}}
\newcommand{\zeros}{{\mathbf{0}}}

\usepackage{stmaryrd}
\usepackage{amsmath}
\usepackage{amsthm}
\usepackage{amssymb}
\usepackage{bbm}
\usepackage{bm}
\usepackage{soul}
\usepackage{multirow}

\usepackage[capitalize,noabbrev]{cleveref}

\theoremstyle{plain}
\newtheorem{theorem}{Theorem}[section]
\newtheorem{proposition}[theorem]{Proposition}
\newtheorem{lemma}[theorem]{Lemma}

\theoremstyle{definition}

\theoremstyle{remark}
\newtheorem{remark}[theorem]{Remark}

\def\c{\mathsf{c}}

\usepackage[utf8]{inputenc} 
\usepackage[T1]{fontenc}    
\usepackage{booktabs}       
\usepackage{amsfonts}       
\usepackage{nicefrac}       
\usepackage{microtype}      
\usepackage{xcolor}         

\usepackage{tikz}
\usepackage{quiver}
\usetikzlibrary{positioning, calc, tikzmark}

\usepackage{placeins}

\title{Quantization-based Bounds on the Wasserstein Metric}

\author{%
  Jonathan Bobrutsky \\
  Department of Statistics and Operations Research\\
  Tel Aviv University\\
  \texttt{jbobrutsky@gmail.com} \\
  \AND
  Amit Moscovich \\
  Department of Statistics and Operations Research\\
  Tel Aviv University\\
  \texttt{mosco@tauex.tau.ac.il}}

\begin{document}

\maketitle

\begin{abstract}
The Wasserstein metric has become increasingly important in many machine learning applications such as generative modeling, image retrieval and domain adaptation. Despite its appeal, it is often too costly to compute. This has motivated approximation methods like entropy-regularized optimal transport, downsampling, and subsampling, which trade accuracy for computational efficiency.  In this paper, we consider the challenge of computing efficient approximations to the Wasserstein metric that also serve as strict upper or lower bounds.  Focusing on discrete measures on regular grids, our approach involves formulating and exactly solving a Kantorovich problem on a coarse grid using a quantized measure and specially designed cost matrix, followed by an upscaling and correction stage.  This is done either in the primal or dual space to obtain valid upper and lower bounds on the Wasserstein metric of the full-resolution inputs.  We evaluate our methods on the DOTmark optimal transport images benchmark, demonstrating a 10×–100× speedup compared to entropy-regularized OT while keeping the approximation error below 2\%.
\end{abstract}

\section{Introduction} \label{sec:introduction}
The Wasserstein metric is a basic tool in machine learning with widespread adoption in various domains, including computer vision, natural language processing, and computational biology \citep{arjovskyWassersteinGenerativeAdversarial2017, kusnerWordEmbeddingsDocument2015,schiebingerOptimalTransportAnalysisSingleCell2019}.
While it is fast to compute in particular cases, such as one-dimensional distributions, in general settings computing the Wasserstein metric can be very expensive.
For example, consider the calculation of the Wasserstein metric between $N \times N$ images that we treat as discrete measures on a regular grid.
Each image has $N^2$ pixels, so computing the Wasserstein metric involves solving a linear program with $N^2 \times N^2$ variables and $\Theta(N^2)$ constraints. This is typically done using a network simplex algorithm whose worst-case runtime in this case is $O(N^6 \log N)$ \citep{peyreComputationalOptimalTransport2019}. For 3D signals, the computational cost is even worse at $O(N^9 \log N)$.
This severely limits the adoption of the Wasserstein metric across many domains and in particular for 2D or 3D discrete signals.
As a result of this, many authors have developed fast approximations of the Wasserstein metric \citep{
cuturiSinkhornDistancesLightspeed2013, 
DeshpandeZhangSchwing2018, 
gerberMultiscaleStrategiesComputing2017, 
shirdhonkarApproximateEarthMovers2008}. 

We consider the challenge of designing approximations that also serve as strict bounds on the $p$-Wasserstein metric between discrete measures on a regular grid.
In this paper we introduce several efficient algorithms that are based on quantization (or downscaling) of the inputs onto a coarse grid and sum-pooling the measures.
We then construct a particular cost matrix for the coarse grid that takes the original measures into account.
This is followed by a correction stage that upscales the solution on the coarse grid to a solution on the original grid and corrects the marginals using iterative proportional-fitting 
The upscaling and correction stage is done separately in the primal and dual spaces, to obtain both upper and lower bounds (respectively).
Finally, to guarantee the correctness of the bounds without relying on the convergence of the proportional fitting procedure, we introduce an additional total variation correction term.

We developed an efficient GPU implementation of the algorithms described in this paper using the JAX package and tested it on the DOTmark optimal transport benchmark \citep{schrieberDOTmarkBenchmarkDiscrete2017}.
We demonstrate that our algorithms achieve significant computational speedups compared to a baseline derived from the entropy-regularized OT. Up to 10$\times$ for upper bounds and up to 100$\times$ speed up for the lower bounds, while maintaining a low approximation error (<2\%) relative to the exact Wasserstein distance.

\subsection{Related work}\label{sec:related_work}
In recent years, the efficient computation of the Wasserstein distance has been a focal point of research, leading to several innovative approaches. The entropy-regularized Sinkhorn algorithm by \cite{cuturiSinkhornDistancesLightspeed2013} remains a cornerstone, offering significant computational speedups despite introducing bias. \cite{AltschulerWeedRigollet2018} advanced this with near-linear time approximation algorithms. Multiscale methods, as discussed by \cite{gerberMultiscaleStrategiesComputing2017} and \cite{merigotMultiscaleApproachOptimal2011}, employ hierarchical strategies to enhance computational efficiency. The benefits of the multiscale approach demonstrated by \cite{feydyFastScalableOptimal2021} for analyzing brain tractograms by adapting the Sinkhorn algorithm. Other quantization-based methods, notably by \cite{beugnotImprovingApproximateOptimal2021}, improve approximation performance, for sampled continuous measures. For grid data, \cite{solomonConvolutionalWassersteinDistances2015} and \cite{chenComputingWasserstein$p$Distance2022} exploit structural advantages for speed. In machine learning, \cite{CourtyFlamaryTuiaRakotomamonjy2017} and \cite{Alvarez-MelisFusi2020a} demonstrate optimal transport's versatility in domain adaptation. \cite{MontesumaMboulaSouloumiac2025} offer a comprehensive survey of recent advances, underscoring the method's growing impact.

\section{Background}\label{sec:background}
\paragraph{Notation}
We denote the non-negative real numbers by $\R_+$ and the set of integers $\{1,\dots,n\}$ by $[n]$. The tensor product is denoted by $\otimes$ whereas pointwise multiplication and division are denoted by $\odot$ and $\oslash$, respectively. The $L^p$ norm of a vector is $\|\cdot\|_p$.
The all-ones column vector is $\ones_n \in \R^n$. The standard vector inner product is denoted by $\inner{\cdot,\cdot}$ and we use the same notation for the inner products of matrices.
The support of a matrix $A \in \R^{n \times m}$ is the set of indices of nonzero elements $\supp{A} = \{(i,j) \in [n] \times [m] \mid A_{i,j} \neq 0 \}$. The cardinality of a set $S$ is denoted by $\#S$. $\delta_x$ is the Dirac delta function at point $x$. Complete list of notations used in the paper is provided in \cref{tbl: notations}.

\subsection{Optimal Transport}
In the following, we give a quick review of basic concepts from optimal transport and refer the reader to \citet{peyreComputationalOptimalTransport2019} for a more thorough introduction.
Consider two discrete probability measures $\mu$, $\nu$ with point masses at $\calX = \{x_1,\dots,x_n\}$ and $\calY = \{y_1,\dots,y_m\}$ respectively. We can express the measures as a sum of Dirac delta functions,
\begin{equation}
    \mu = \sum_{i=1}^n \mu_i \delta_{x_i},  \qquad \nu = \sum_{j=1}^m \nu_j \delta_{y_j}\ .
\end{equation}
We identify the measures with their non-negative coefficient vectors $\vmu\in\R_+^n,\vnu\in\R_+^m$.
Since $\mu$ and $\nu$ are probability measures, we must have $\vmu \in \Sigma_n$ and $\vnu \in \Sigma_m$ where
\begin{align}
    \Sigma_N
    :=
    \big\{
    (p_1, \ldots, p_N) \in \R_+^N: \, p_1 + \cdots + p_N = 1
    \big\}
\end{align}
is the probability simplex.
The set of \emph{coupling} matrices between $\mu$ and $\nu$ is \begin{align}
    \Pi(\vmu,\vnu)
    :=
    \left\{
    \pi\in\R_+^{n \times m}
    \middle|
    \pi\ones_m=\vmu,\,
    \pi^{\top}\ones_n=\vnu
    \right\}.
\end{align}
Each coupling can be viewed as a transport plan between $\mu$ and $\nu$ where $\pi_{ij}$ is the amount of mass transported from $x_i$ to $y_j$.
The marginal constraints $\pi\ones_m=\vmu$ mean that the entire source measure $\mu$ is transported, whereas the constraints $\pi^{\top}\ones_n=\vnu$ mean that this transport results in the target  measure $\nu$.
In particular, the set $\Pi(\vmu,\vnu)$ always contains the trivial coupling $\pi_{\otimes} := \vmu\otimes\vnu$ that distributes every point mass in $\calX$ to all point masses in $\calY$ proportionately to $\vnu$. Let $C \in \R_+^{n \times m}$ be a ground-cost matrix, where  $C_{ij}$ represents the cost of transporting a unit mass from $x_i\in\calX$ to $y_j\in\calY$.
The Kantorovich problem is the minimization problem that seeks a cost-minimizing coupling between $\mu$ and $\nu$,
\begin{align}
    L_C(\mu,\nu) := \min_{\pi\in\Pi(\vmu,\vnu)} \inner{\pi, C}.
\end{align}
This is a linear optimization problem with linear constraints.
It admits a dual program,
\begin{equation}
    L_C(\mu,\nu) = \max_{(\vf, \vg)\in\mathcal{R}(C)} \inner{\vf,\vmu} + \inner{\vg,\vnu},
\end{equation}
where
\begin{align}
    \mathcal{R}(C)
    :=
    \{
    (\vf,\vg)\in\R^n \times \R^m
    |
    \forall i,j: f_i + g_j \le C_{ij}
    \}
\end{align}
is the set of admissible dual potentials, also known as Kantorovich potentials.
Given a distance metric $\rho: \calX \times \calX \rightarrow \R_+$, for any $p \ge 1$ the \emph{Wasserstein}-$p$ metric is a metric over the space of discrete probability measures with point masses at $\calX = \{x_1, \ldots, x_n\}$, defined as
\(
    \Wass_p(\mu,\nu) :=  L_C(\mu,\nu)^{\frac{1}{p}}
\)
where
\(
    C_{ij} = \rho(x_i,y_j)^p.
\)

\subsection{Measure coarsening} \label{sec:coarsening}
Consider a $d$-dimensional regular grid $\calX = [N]^d$.
Suppose that $N$ is a multiple of some scale factor $\kappa \in \mathbb{N}$. In that case we may subdivide the grid along the axes into a set $\sX = \{ X_1,\dots,X_{n^d}\}$ of non-overlapping hypercubes of cardinality $\kappa^d$ that cover the entire grid $\calX$.
We define the coarse grid
$\tilde{\calX} := \{ \tilde{x}_1,\dots,\tilde{x}_{n^d}\}$
as the set of all hypercube centers, with $\tilde{x}_k = \code{mean}(X_k)$. Given a discrete measure $\mu\in\Sigma_{N^d}$ over the grid $\calX$, we define a coarsened discrete measure by placing the mass associated with each hypercube at its center point,
\(
    \tilde{\mu} = \sum_{k=1}^{n^d} \mu(X_k)\delta_{\tilde{x}_k}.
\)
The coarsening of the regular grid corresponds to the \code{SumPool} and \code{AvgPool} operations on the measure and the coordinates, respectively, with size and stride $\kappa$.

\section{Methods}\label{sec:methods}
In this section, we describe several algorithms for computing bounds of the Wasserstein distance $\Wass_p(\mu,\nu)$ on a regular grid.

\subsection{Bounds based on Entropy regularized OT}\label{sec: entropy regulariztion bounds}
Entropy regularized optimal transport, also known as Sinkhorn distance \citep{cuturiSinkhornDistancesLightspeed2013}, adds an entropy term $\mathcal{H}$ to the primal:
\begin{align}
    L_C^\varepsilon(\mu,\nu) := \min_{\pi\in\Pi(\vmu,\vnu)} \inner{\pi, C} - \varepsilon\mathcal{H}(\pi).
\end{align}
This makes the problem strongly convex and solvable using Sinkhorn iterations 
\citep{knoppConcerningNonnegativeMatrices1967}. In this section we will show how computing entropy regularized OT can be used to construct upper and lower bounds on the exact Wasserstein distance.
\paragraph{Lower bound} \label{sec: reg lower bound}
The dual form of the Sinkhorn distance is
\begin{align}
    L_C^\varepsilon(\mu,\nu)
    :=
    \max_{\vf\in\R^n, \vg\in\R^m} \inner{\vf,\vmu}
    + \inner{\vg,\vnu}
    - \varepsilon\inner{e^{\vf/\varepsilon}, K e^{\vg/\varepsilon}} 
\end{align}
where $K_{ij} := e^{-C_{ij}/\varepsilon}$ is the Gibbs kernel. The algorithmic solution, defined by the use of a finite number of iterations $t$ that achieves some stopping criteria, is known to satisfy a lower bound. Summarized in the following proposition.
\begin{proposition}
    Let $\hat\vf_\varepsilon^{(t)}, \hat\vg_\varepsilon^{(t)}$ be the iterations of the Sinkhorn distance algorithm in step $t\in\N$.
    \begin{align}
        \inner{\hat\vf_\varepsilon^{(t)}, \vmu} + \inner{\hat\vg_\varepsilon^{(t)}, \vnu} \le L_C(\mu,\nu)
    \end{align}
    as soon as $t \ge 1$.
\end{proposition}
This follows directly from \citet[Propositions~4.5,4.8]{peyreComputationalOptimalTransport2019}, so we can define $\underline{\Wass}_{p,\varepsilon}(\mu,\nu) := (\inner{\hat\vf_\varepsilon^{(t)}, \vmu} + \inner{\hat\vg_\varepsilon^{(t)}, \vnu})^\frac{1}{p} \le \Wass_p(\mu,\nu)$ for all $p \ge 1$.

\paragraph{Upper bound}\label{sec: reg upper bound}
Consider $d$-dimensional regular grids with side length $N \in \N$, $\calX=\calY = [N]^d$, with discrete measures $\vmu,\vnu\in\Sigma_{N^d}$. Although the converged regularized optimal coupling
\begin{align}
    \pi_\varepsilon^* = \argmin_{\pi\in\Pi(\vmu,\vnu)} \inner{\pi, C} - \varepsilon\mathcal{H}(\pi)
\end{align}
defines an upper bound on the optimal transport $\inner{\pi_\varepsilon^*, C} \ge L_C(\mu,\nu)$, the algorithmic solution $\hat\pi_\varepsilon^{(t)}$ does not. Since the marginals $\hat\vmu_\varepsilon^{(t)}=\hat\pi_\varepsilon^{(t)}\ones_N,\, \hat\vnu_\varepsilon^{(t)}=(\hat\pi_\varepsilon^{(t)})^\top\ones_N$ do not identify with the couplings $\vmu,\vnu$.
We bound the effect of this difference using the weighted total variation. For some $x_0 \in \calX$, using distance weights $\vw=\{\rho(x_0,x)^p\}_{x\in\calX}$ the Wasserstein-$p$ distance is controlled by \emph{weighted} total variation \citep{villaniOptimalTransportOld2009},
\begin{equation}\label{eq: weighted total variation}
    \mathcal{TV}_p^\vw(\mu,\nu) := 2^{1-\frac{1}{p}} \inner{\vw,|\vmu-\vnu|}^{\frac{1}{p}} \ge \Wass_p(\mu,\nu).
\end{equation}
Here $|\cdot|$ is the element-wise absolute value.

Defining the total variation corrected regularization-based upper bound
\begin{equation}\label{eq: reg upper bound def}
    \overline{\Wass}_{p,\varepsilon}(\mu,\nu)
    :=
    \inner{\hat\pi_\varepsilon^{(t)}, C}^\frac{1}{p} + \Delta_{\hat{\mu}_\varepsilon^{(t)}} +
    \Delta_{\hat{\nu}_\varepsilon^{(t)}} ,
\end{equation}
where $\Delta_{\hat{\mu}_\varepsilon^{(t)}} = \mathcal{TV}_p^\vw(\hat{\mu}_\varepsilon^{(t)}, \mu),\,
    \Delta_{\hat{\nu}_\varepsilon^{(t)}} = \mathcal{TV}_p^\vw(\nu, \hat{\nu}_\varepsilon^{(t)})$ are the marginal corrections
with weights $\vw=\{\rho(\bar{x},x_i)\}_i$ are taken from the center of the measure $\bar{x}=\code{mean}(\calX)$.
Using the triangle inequality for the Wasserstein metric, we can write
\begin{lemma}\label{lm: triangle inequality upper bound}
    Let $\mu,\hat{\mu},\nu,\hat{\nu} \in \Sigma_{N^d}$ discrete measures on $\calX=[N]^d$, and $\hat{\pi} \in \Pi(\hat{\mu},\hat{\nu})$ is a coupling between $\hat{\mu}$ and $\hat{\nu}$. For $p \ge 1$,
    \begin{align}
        \Wass_p(\mu,\nu) \le  \inner{\hat{\pi},C}^{\frac{1}{p}} + \Delta_{\hat\mu} + \Delta_{\hat\nu}
    \end{align}
\end{lemma}
Combining this with \eqref{eq: reg upper bound def} shows that
\(
     \Wass_{p}(\mu,\nu) \le \overline{\Wass}_{p,\varepsilon}(\mu,\nu).
\)
\begin{proposition}\label{th: $L^1$ upper bound on weighted total variation}
    Let $\mu,\hat{\mu},\nu,\hat{\nu} \in \Sigma_n$ be discrete measures in $\calX$ and $\xi > 0$, satisfying the convergence criteria $\|\hat\vmu - \vmu\|_1 + \|\vnu - \hat{\vnu}\|_1 <  \xi$. The sum of the marginal corrections is bounded,
    \begin{equation}
        \Delta_{\hat\mu} + \Delta_{\hat\nu}
        < 2^{2-\frac{2}{p}} \xi^\frac{1}{p}r
    \end{equation}
    where $r:=\max_{x\in\calX}\{\rho(\bar{x},x)\}$ is radius of  $\calX$.
\end{proposition}
Proofs for \cref{lm: triangle inequality upper bound,th: $L^1$ upper bound on weighted total variation} are provided in \cref{sc: proof of triangle inequality upper bound}.

\begin{figure}[t]
    \centering
    \begin{minipage}{0.5\textwidth}
        \centering
        \includegraphics[width=\textwidth]{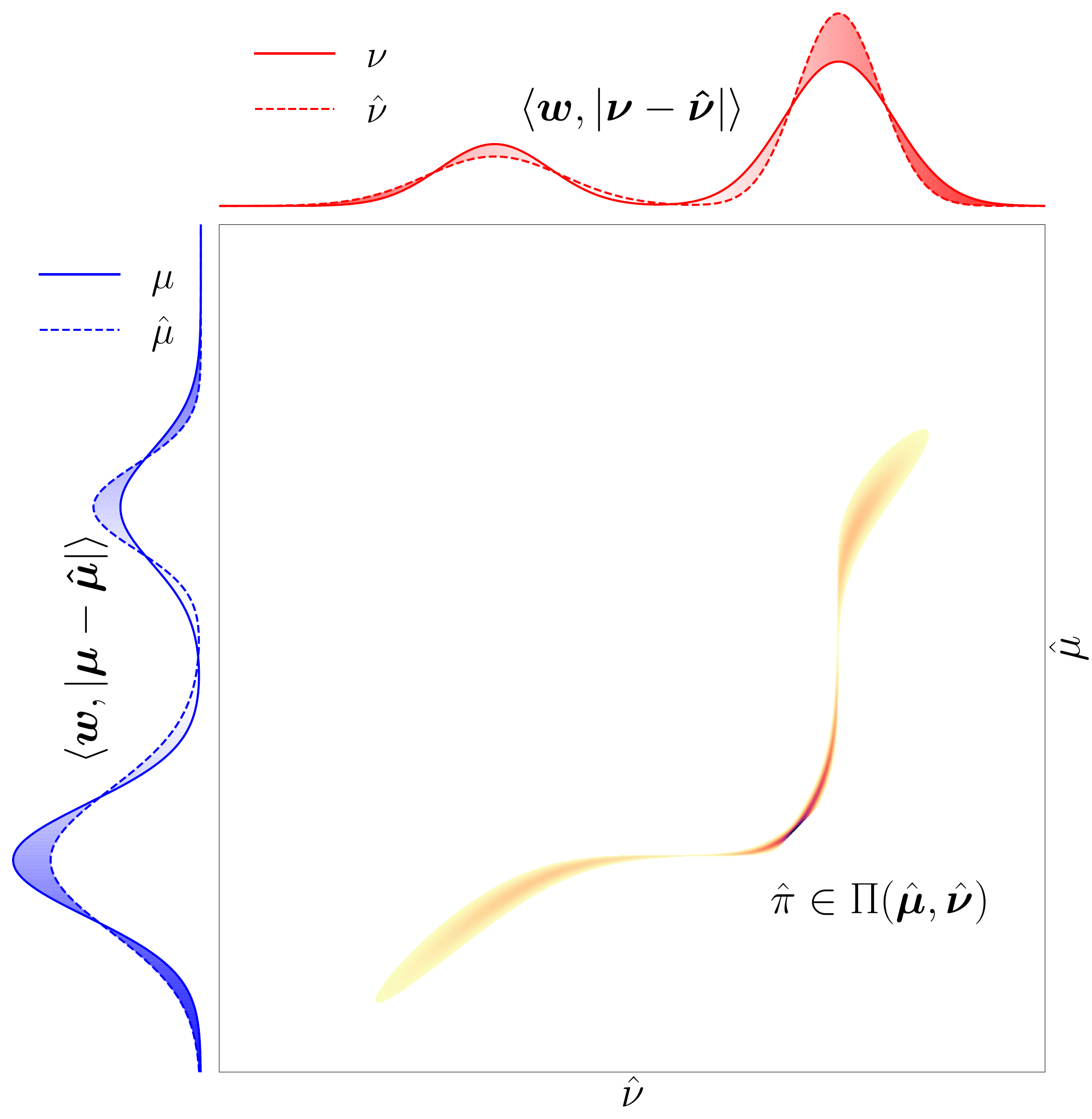}
    \end{minipage}
    \qquad\qquad\qquad
    \begin{minipage}{0.3\textwidth}
        \vspace{3em}
        \centering
        \begin{tikzcd}[ampersand replacement=\&]
            \textcolor{blue}{\mu} \\
            \&\& \textcolor{blue}{{\hat\mu}} \\
            \\
            \\
            \&\& \textcolor{red}{{\hat\nu}} \\
            \textcolor{red}{\nu}
            \arrow["{\Wass_p(\mu,\hat\mu)}"{description}, color=blue, from=1-1, to=2-3]
            \arrow["{\le \mathcal{TV}_p^\vw(\mu, \hat\mu)}"{description, pos=0.6}, shift left=4, color=blue, curve={height=-12pt}, dashed, from=1-1, to=2-3]
            \arrow["{\Wass_p(\mu,\nu)}"{description}, color=teal, Rightarrow, from=1-1, to=6-1]
            \arrow["{\Wass_p(\hat\mu,\hat\nu)}"{description, pos=0.4}, color=teal, from=2-3, to=5-3]
            \arrow["{\le \inner{\hat\pi, C}^\frac{1}{p}}"{description, pos=0.6}, shift left=3, color=teal, curve={height=-12pt}, dashed, from=2-3, to=5-3]
            \arrow["{\Wass_p(\hat\nu,\nu)}"{description}, color=red, from=5-3, to=6-1]
            \arrow["{\le \mathcal{TV}_p^\vw(\hat\nu,\nu)}"{description, pos=0.4}, shift left=3, color=red, curve={height=-12pt}, dashed, from=5-3, to=6-1]
        \end{tikzcd}
    \end{minipage}
    \caption{Optimal transport bounds visualization. Left: Illustration of optimal transport between discrete probability measures. Right: Diagram showing the relationship between measures and their Wasserstein distances.}
    \label{fig:visualization}
\end{figure}
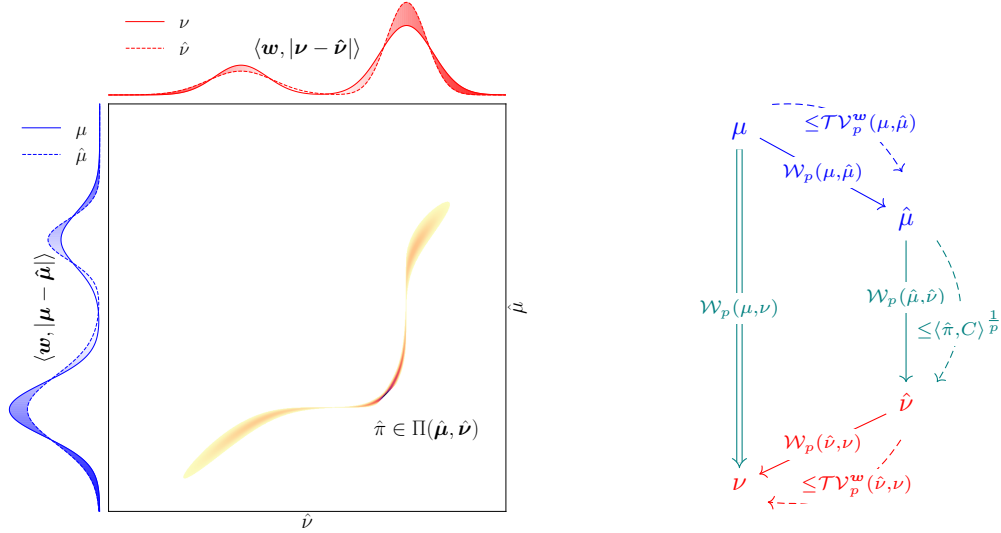

\subsection{Weighted-cost upper bound}\label{sec: Weighted-Cost Upper Bound}
In this subsection we consider an approach to bound Wasserstein distance by down-scaling the grid $\calX$ to $\tilde{\calX}=\tilde{\calY}$ and the measures to $\tilde\vmu,\tilde\vnu \in \Sigma_{n^d}$ using regular hypercubes as described in \cref{sec:coarsening}. We define the marginally weighted coarse cost

\begin{align}\label{eq: marginally weighted cost}
    \bar{C}_{k\ell}
    :=
    \frac{1}{\mu(X_k)\nu(Y_\ell)} \sum_{\substack{
    x\in X_k, y\in Y_\ell}}
    \rho(x,y)^p \mu(x)\nu(y).
\end{align}
It follows that $\bar{C} = \code{SumPool}(C\odot\pi_\otimes, \kappa) \oslash \code{SumPool}(\pi_\otimes, \kappa)$, which can be used efficiently for small enough grids. Then we compute the optimal coupling for the marginally weighted coarse cost using network simplex solver \citep{bonneelDisplacementInterpolationUsing2011}, defining an upper bound
\begin{align}
    \overline{\Wass}_p^{\otimes} := L_{\bar{C}}(\tilde{\mu},\tilde{\nu})^\frac{1}{p}.
\end{align}

\begin{theorem}\label{th: weighted cost UB}
    The optimal transport loss under the marginally weighted coarse cost $L_{\bar{C}}(\tilde{\mu},\tilde{\nu})$ is an upper bound to the optimal transport loss $L_C(\mu,\nu)$, and similarly for the Wasserstein distance,
    \begin{equation}
        L_{\bar{C}}(\tilde{\mu},\tilde{\nu})^\frac{1}{p} \ge \Wass_p(\mu,\nu).
    \end{equation}
\end{theorem}
The proof uses an auxiliary coupling transferring the mass between each pair of sub-regions $(X_i, Y_j)$ using the trivial coupling, weighted by a coarse coupling, choosing the optimal coarse coupling. A detailed proof is provided in \cref{sc: proof of weighted cost UB}.

\subsection{Min-cost lower bound}\label{sec: Min-Cost Lower Bound}
For the same coarsening, we define the locally minimal cost
\(
    C_{k\ell}^{\min} := \min_{x\in X_k,y\in Y_\ell} \rho(x,y)^p,
\)
and compute the optimal coupling for this coarse cost using a network simplex solver, yielding a lower bound
\(
    \underline{\Wass}_p^{\min} := L_{C^{\min}}(\tilde{\mu},\tilde{\nu})^\frac{1}{p}.
\)
The following theorem is proved in \cref{sec: additional proofs}.
\begin{theorem}\label{th: locally-minminal cost LB}
    The coarse optimal transport cost set by the locally minimal cost, is a lower bound of the optimal transport.
    \(
        L_{C^{\min}}(\tilde{\mu},\tilde{\nu}) \le L_C(\mu,\nu).
    \)
\end{theorem}

\subsection{Primal upscaling upper bound}\label{sec: Upscaling upper bound}
In this approach we upscale an optimal coupling for the coarse cost $\tilde{c}_{k\ell} = \rho(\tilde{x}_k,\tilde{y}_\ell)^p$ computed using a network simplex solver $\tilde{\pi}^* = \argmin_{\tilde{\pi} \in \Pi(\tilde{\vmu}, \tilde{\vnu})}\inner{\tilde{\pi},\tilde{C}}$ back to the original problem size.
\paragraph*{Up-scaled coupling}\label{p: step 1}
A coupling matrix $\pi$ of dimensions $n^d \times n^d$ can equivalently be represented as a $2d$-tensor of the shape \( n \times n \times \dots \times n \).
We formally define operations for reshaping matrices into tensors and back.
Let \code{reshape} be a cardinality-preserving transformation from A to $\mathbf{B}$ such that $a_{ij}=b_{u_1,\dots,u_d,\,v_1,\dots,v_d}$
where $(u_1, \dots, u_d)$ is a multi-index that corresponds to a matrix row by $i = 1 + \sum_{k=1}^d u_k n^{k-1}$.
Similarly the multi-index $(v_1, \dots, v_d)$ corresponds to the column index $j$.
Reshaping the coarse optimal coupling $\tilde{\pi}^*$ into a $2d$-tensor $\tilde{\tP}^*$,
we up-scale the optimal coupling using a normalized positive-valued kernel  $\tK$, a $2d$-tensor representing a hypercube of width $\kappa$.
\begin{equation}\label{eq: coupling tensor upscale}
    \hat{\tP} := \tilde{\tP}^* \otimes \tK.
\end{equation}
By inversely reshaping the up-scaled tensor $\hat{\tP}$ into an up-scaled matrix $\hat{\pi}$, we obtain the approximate up-scaled coupling. Using a uniform kernel $\tK$ is equivalent to nearest-neighbor interpolation.
\begin{lemma}\label{lm: normalized coupling}
    Using a normalized positive-valued kernel $\tK$ ensures $\hat{\pi}$ represents a coupling. Satisfying $\hat\pi \in \R_+^{N^d\times N^d}$ and $\sum_{i=1}^{N^d}\sum_{j=1}^{N^d} \hat{\pi}_{ij} = 1$, such that $\hat{\pi} \in \Pi(\hat\pi\ones_{N^d}, \hat\pi^{\top}\ones_{N^d})$.
\end{lemma}

\paragraph*{Iterative proportional fitting}\label{p: step 2}
The approximate up-scaled coupling is $\xi$-fitted into $\hat{\pi}_\xi = \diag(\va)\hat{\pi}\diag(\vb)$ by iterative proportional fitting using Sinkhorn's theorem , until the marginals  $\hat{\vmu}_\xi = \va \odot (\hat{\pi} \vb),\, \hat{\vnu}_\xi = \vb \odot (\hat{\pi}^\top \va)$ are converged to $\|\hat{\vmu}_\xi - \vmu\|_1 + \|\vnu - \hat{\vnu}_\xi\|_1 <  \xi$, where $\va, \vb \in \R_+^{N^d}$ are the vector scale factors.
yielding the approximation
\(
    \widehat{\Wass}_p(\mu,\nu) := \inner{\hat{\pi}_\xi, C}^{\frac{1}{p}}.
\)

\paragraph*{Upper bound}\label{p: step 3}
Using weighted total variation (\cref{eq: weighted total variation}) we define an upscaling upper bound
\begin{align}
    \overline{\Wass}_{p}(\mu,\nu)
    :=
    \widehat{\Wass}_p(\mu,\nu)
    + \Delta_{\hat{\mu}_\xi}
    + \Delta_{\hat{\nu}_\xi} ,
\end{align}

\begin{theorem}[Upscaling Upper Bound]\label{th: Upscaling upper bound}
    $\overline{\Wass}_p(\mu,\nu)$ is upper bound of the Wasserstein distance,
    \begin{align}
        \Wass_p(\mu,\nu) \le \overline{\Wass}_p(\mu,\nu).
    \end{align}
\end{theorem}

\begin{proof}[Proof of \cref{th: Upscaling upper bound}]
    The up-scaled matrix $\hat{\pi} \in \Pi(\hat\pi\ones_{N^d}, \hat\pi^{\top}\ones_{N^d})$ is normalized as a coupling, by \cref{lm: normalized coupling}, and $\hat{\pi}_\xi$ retains this normalization, by Sinkhorn's theorem.  $\overline{\Wass}_p(\mu,\nu)$ is than an upper bound of the Wasserstein distance, by \cref{lm: triangle inequality upper bound}.
\end{proof}

\begin{remark}
    Considering $\calX = [N]^d$ with $L^2$ norm, the radius becomes $r=\frac{1}{2}d^{\frac{1}{2}}N$ (see \cref{th: $L^1$ upper bound on weighted total variation}) such that the weighted total variation correction $\Delta_{\hat\mu} + \Delta_{\hat\nu}$ is at most $2^{1-\frac{2}{p}}d^{\frac{1}{2}}N\xi^{\frac{1}{p}}$. Negligible for $\xi \ll N^{-p}$, and can be ignored for many practical use cases.
\end{remark}

\subsection{Dual upscaling lower bound}\label{sec: Upscaling Lower Bound}
We construct a lower bound for the Wasserstein distance $\Wass_p(\mu,\nu)$ by solving a down-scaled optimal transport problem using coarsened measures.
The coarse optimal Kantorovich potentials are than up-scaled using a multi-linear interpolation and improved using a c-transform.
Considering the same setting as in \cref{sec: Upscaling upper bound}, we solve for the optimal potentials of the down-scaled discrete measures
\begin{align}
    (\tilde{\vf}^*,\tilde{\vg}^*) = \argmax_{(\tilde{\vf},\tilde{\vg}) \in \mathcal{R}(\tilde{C})} \inner{\tilde{\vf}, \tilde\vmu} + \inner{\tilde{\vg}, \tilde\vnu}
\end{align}
and evaluate the dual transport cost at the original scale by up-scaling the optimal potentials. Up-scaling can be performed by any multivariate interpolation method such as nearest-neighbor, spline, multi-linear and polynomial methods.
Using an interpolation function
$R: \calX \cup \tilde{\calX} \rightarrow \R$,
the up-scaled potential
$\hat{\vf}$ is defined by
\begin{align}
    \hat{\vf} := \{R_{\tilde{\vf}, \tilde{\calX}}(x_i)\}_{i\in[N^d]}.
\end{align}
An important property of the dual formulation is that for every potential $\vf \in \R^n$ we can easily find a \emph{tight} potential $\vf^\c \in \R^m$ such that $(\vf,\vf^\c) \in \mathcal{R}(C)$, by
\(
    f_j^\c := \min_i C_{ij}-f_i \,.
\)
This is known as a \emph{c-transform}. It can be shown that repeating this process once more achieves a tight pair
$(\vf^\c, \vf^{\c\c}) \in \mathcal{R}(C)$, where
$f_i^{\c\c} := \min_j C_{ij}-f_j^\c$.
Thus, a lower bound is guaranteed by using the c-transform to generate the potential pair from the up-scaled potential
$\vf = \hat{\vf}^{\c\c}$ and $\vg = \hat{\vf}^{\c}$,
which yields the upscaling lower bound
\begin{equation}\label{eq: upscaling lower bound}
    \underline{\Wass}_p(\mu, \nu) := \left(\inner{\vf,\vmu} + \inner{\vg,\vnu} \right)^{\frac{1}{p}}.
\end{equation}

Since by the admissibility of a c-transformed pair $(\vf, \vg) \in \mathcal{R}(C)$, we can write
\begin{proposition}
    Considering an approximate potential $\hat{\vf} \in \R^{N^d}$. For $\vf = \hat{\vf}^{\c\c}$ and $\vg = \hat{\vf}^{\c}$
    \begin{align}
        \inner{\vf,\vmu} + \inner{\vg,\vnu} \le L_C(\mu, \nu).
    \end{align}
\end{proposition}

\section{Computational complexity analysis}\label{sec:complexity}
The quantization-based bounds involve the following steps: computing the Wasserstein metric on the quantized measures, upscaling the dual potentials or couplings to the original scale, and the calculation of weighted total variation correction terms.
The latter is calculated in linear time and space, thus negligible w.r.t. the other steps. In the following, we detail the computational gains provided by the proposed methods.

\paragraph{Downscaled optimal transport}
The solution to the Kantorovich problem of the scaled measures can be solved by dedicated linear programming methods, such as the network  simplex used in \citet{bonneelDisplacementInterpolationUsing2011}, with $O\left(n^{3d}\log{n}\right)$  time complexity \citep{ahujaNetworkFlowsTheory1993}. By solving only for the optimal transport of the coarse measures, we produce a computational speedup of \(\Theta(\kappa^{3d})\) (up to log factors).
The space complexity can also be significantly reduced, since one can avoid storing the full cost matrix of size $N^d \times N^d$, by using coarse cost matrices, e.g. $\bar{C},C^{\min}$, and $\tilde{C}$ of size $n^d\times n^d$, realizing a memory gain of $\Theta(\kappa^{2d})$.

\paragraph{Upscaled optimal coupling}
The optimal coarse coupling $\tilde{\pi}^*$ is a sparse matrix with at most $2n^d - 1$ positive entries \citep[Proposition~3.4]{peyreComputationalOptimalTransport2019}. Thus, the up-scaled approximate coupling from \cref{eq: coupling tensor upscale} conserves this sparsity with $\#\supp{\hat{\pi}} \le \kappa^{2d}(2n^d - 1)$, allowing to calculate the approximate optimal transport
\begin{equation}
    \inner{\hat{\pi}_\xi, C} = \sum_{(i, j) \in \supp{\hat{\pi}}}  \overbrace{a_i\hat{\pi}_{ij}b_j}^{(\hat{\pi}_\xi)_{ij}} \, \rho(x_i, y_j)^p.
\end{equation}
without impacting the total time and space complexity of the coarse optimal transport solution.
\begin{equation}
    \frac{\#{C}}{\#\supp{\hat{\pi}}} = \frac{N^{2d}}{\kappa^{2d}(2n^d - 1)}
    =
    \Theta\big( \left( N/\kappa \right)^d \big).
\end{equation}

\paragraph{Lazy c-transform}
To reduce the memory requirements of \cref{eq: upscaling lower bound}, we evaluate the c-transform on-demand (i.e. "lazy") without storing the entire cost matrix $C$.
\begin{align}\label{eq: lazy c-transform}
    g_j & \leftarrow \min_{i} \rho(x_i,y_j)^p - \hat{f}_i \\
    f_i & \leftarrow \min_{j} \rho(x_i,y_j)^p - g_j
\end{align}

\begin{table}[h]
    \begin{tabular}{lcc}
        \toprule
        \textbf{Method} & \textbf{Time Complexity} & \textbf{Space Complexity} \\
        \midrule
        Entropic Regularization-Based Bounds \citep{LinHoJordan2022} & $\tilde{O}\left((n \kappa)^{2d}/\varepsilon^2\right)$& $O\left((n \kappa)^{2d}\right)$\\
        Quantization-Based Bounds & $\tilde{O}\left(n^{3d} \right)$ & $O\left(n^{2d}\right)$ \\
        \bottomrule
    \end{tabular}
\caption{Complexity of different bounds in terms of the fine-scale cardinality $\#\calX=N^d=(n \kappa)^d$. The first row corresponds to the methods in \cref{sec: entropy regulariztion bounds} and the second to \crefrange{sec: Weighted-Cost Upper Bound}{sec: Upscaling Lower Bound}.}\label{tbl: complexity}\vspace{-2pt}
\end{table}

\section{Experiments}\label{sec:experiments}
The algorithms were implemented in Python and optimized for GPU acceleration using JAX numerical computation library \citep{jax2018github}. For entropy regularized optimal transport we used OTT-JAX \citep{CuturiMeng-PapaxanthosTianBunneDavisEtAl2022} and POT \citep{flamaryPOTPythonOptimal2021} for exact optimal transport. The methods were benchmarked on a machine using an NVIDIA L40 GPU and AMD EPYC 9654 CPU.

\paragraph{DOTmark}
The methods were evaluated on 2D images from the discrete optimal transport  benchmark \citep{schrieberDOTmarkBenchmarkDiscrete2017}, using $\rho=L^2$ the euclidean metric, at $p=\{1,2\}$. To examine the effect of the scaling factor, the quantization-based methods were evaluated using $\kappa=\{2,4\}$. To examine the effect of the entropic-regularization parameter, the regularization-based methods were evaluated using $\varepsilon=\{0.001N^p,0.004N^p\}$ explicitly dependent on $N^p$ term to avoid large $\|C\|_\infty / \varepsilon$ causing numerical instability \citep{AltschulerWeedRigollet2018}, since $\|C\|_\infty \propto N^p$ in our setting.
For upper bounds, while at $p=1$ the entropic-regularization upper bound at $\varepsilon=0.001N^p$ delivers the best approximation, summarized in \cref{tbl: accuracy}, it does so with significant impact on the computation time, as seen in \cref{fig: efficiency of wasserstein upper bounds}. Otherwise, both for upper and lower bound the quantization methods produce the best approximations at $\kappa=2$, while computing with relative time second only to the quantization methods scaled at $\kappa=4$.

\begin{table}[h]
    \centering
    \small
    \setlength{\tabcolsep}{2pt}
    \begin{tabular}{lccccccc|cccccc}
    \toprule
     & & \multicolumn{6}{c}{Upper Bounds} & \multicolumn{6}{c}{Lower Bounds} \\
    \cmidrule{3-8}\cmidrule{9-14}
     & & \multicolumn{2}{c}{\shortstack{Weighted-\\Cost}} & \multicolumn{2}{c}{\shortstack{Primal\\Upscaling}} & \multicolumn{2}{c}{\shortstack{Entropic\\Regularization}} & \multicolumn{2}{c}{\shortstack{Dual\\Upscaling}} & \multicolumn{2}{c}{\shortstack{Min-\\Cost}} & \multicolumn{2}{c}{\shortstack{Entropic\\Regularization}} \\
     & & $\kappa_2$ & $\kappa_4$ & $\kappa_2$ & $\kappa_4$ & $\varepsilon_1$ & $\varepsilon_4$ & $\kappa_2$ & $\kappa_4$ & $\kappa_2$ & $\kappa_4$ & $\varepsilon_1$ & $\varepsilon_4$ \\
    Class & p &  &  &  &  &  &  &  &  &  &  &  &  \\
    \midrule
    \multirow[t]{4}{*}{Classic} & \multirow[t]{2}{*}{1} & 3.1\% & 11.0\% & 9.6\% & 23.0\% & \textbf{0.9\%} & 5.2\% & \textbf{0.3\%} & 0.7\% & 10.0\% & 27.0\% & 24.0\% & 88.0\% \\ Images
     &  & {\scriptsize $\pm$ 2.0\%} & {\scriptsize $\pm$ 7.1\%} & {\scriptsize $\pm$ 4.1\%} & {\scriptsize $\pm$ 9.9\%} & {\scriptsize $\pm$ 0.5\%} & {\scriptsize $\pm$ 2.1\%} & {\scriptsize $\pm$ 0.2\%} & {\scriptsize $\pm$ 0.4\%} & {\scriptsize $\pm$ 6.4\%} & {\scriptsize $\pm$ 16.0\%} & {\scriptsize $\pm$ 8.3\%} & {\scriptsize $\pm$ 15.0\%} \\
    \cline{2-14}\noalign{\vskip 1.5pt}
     & \multirow[t]{2}{*}{2} & \textbf{1.6\%} & 7.9\% & 2.2\% & 8.8\% & 14.0\% & 44.0\% & \textbf{0.7\%} & 2.4\% & 13.0\% & 33.0\% & 98.0\% & 100.0\% \\
     &  & {\scriptsize $\pm$ 1.2\%} & {\scriptsize $\pm$ 5.2\%} & {\scriptsize $\pm$ 1.4\%} & {\scriptsize $\pm$ 5.5\%} & {\scriptsize $\pm$ 8.8\%} & {\scriptsize $\pm$ 25.0\%} & {\scriptsize $\pm$ 0.5\%} & {\scriptsize $\pm$ 1.6\%} & {\scriptsize $\pm$ 3.7\%} & {\scriptsize $\pm$ 8.8\%} & {\scriptsize $\pm$ 8.4\%} & {\scriptsize $\pm$ 0.0\%} \\
    \midrule
    \multirow[t]{4}{*}{Micro-} & \multirow[t]{2}{*}{1} & 0.9\% & 3.4\% & 2.4\% & 6.5\% & \textbf{0.4\%} & 2.0\% & \textbf{0.4\%} & 0.9\% & 6.2\% & 17.0\% & 9.6\% & 38.0\% \\ scopy
     &  & {\scriptsize $\pm$ 1.7\%} & {\scriptsize $\pm$ 5.9\%} & {\scriptsize $\pm$ 3.2\%} & {\scriptsize $\pm$ 8.4\%} & {\scriptsize $\pm$ 0.3\%} & {\scriptsize $\pm$ 2.0\%} & {\scriptsize $\pm$ 0.5\%} & {\scriptsize $\pm$ 0.9\%} & {\scriptsize $\pm$ 4.2\%} & {\scriptsize $\pm$ 10.0\%} & {\scriptsize $\pm$ 7.1\%} & {\scriptsize $\pm$ 22.0\%} \\
    \cline{2-14}\noalign{\vskip 1.5pt}
     & \multirow[t]{2}{*}{2} & \textbf{0.5\%} & 2.2\% & 0.7\% & 2.7\% & 3.8\% & 11.0\% & \textbf{0.2\%} & 0.7\% & 5.5\% & 16.0\% & 30.0\% & 90.0\% \\
     &  & {\scriptsize $\pm$ 0.7\%} & {\scriptsize $\pm$ 3.5\%} & {\scriptsize $\pm$ 1.0\%} & {\scriptsize $\pm$ 3.9\%} & {\scriptsize $\pm$ 5.9\%} & {\scriptsize $\pm$ 18.0\%} & {\scriptsize $\pm$ 0.4\%} & {\scriptsize $\pm$ 1.2\%} & {\scriptsize $\pm$ 3.4\%} & {\scriptsize $\pm$ 8.4\%} & {\scriptsize $\pm$ 32.0\%} & {\scriptsize $\pm$ 19.0\%} \\
    \midrule
    \multirow[t]{4}{*}{Shapes} & \multirow[t]{2}{*}{1} & 1.1\% & 3.6\% & 3.2\% & 7.8\% & \textbf{0.7\%} & 2.6\% & \textbf{0.5\%} & 1.0\% & 7.3\% & 20.0\% & 13.0\% & 51.0\% \\
     &  & {\scriptsize $\pm$ 2.2\%} & {\scriptsize $\pm$ 4.8\%} & {\scriptsize $\pm$ 2.6\%} & {\scriptsize $\pm$ 6.3\%} & {\scriptsize $\pm$ 2.5\%} & {\scriptsize $\pm$ 2.2\%} & {\scriptsize $\pm$ 2.8\%} & {\scriptsize $\pm$ 3.0\%} & {\scriptsize $\pm$ 5.2\%} & {\scriptsize $\pm$ 11.0\%} & {\scriptsize $\pm$ 8.1\%} & {\scriptsize $\pm$ 22.0\%} \\
    \cline{2-14}\noalign{\vskip 1.5pt}
     & \multirow[t]{2}{*}{2} & \textbf{1.2\%} & 3.2\% & 1.4\% & 3.6\% & 5.1\% & 15.0\% & \textbf{0.9\%} & 1.7\% & 7.7\% & 20.0\% & 52.0\% & 99.0\% \\
     &  & {\scriptsize $\pm$ 4.5\%} & {\scriptsize $\pm$ 5.0\%} & {\scriptsize $\pm$ 4.5\%} & {\scriptsize $\pm$ 5.0\%} & {\scriptsize $\pm$ 6.2\%} & {\scriptsize $\pm$ 17.0\%} & {\scriptsize $\pm$ 5.1\%} & {\scriptsize $\pm$ 5.6\%} & {\scriptsize $\pm$ 6.1\%} & {\scriptsize $\pm$ 9.2\%} & {\scriptsize $\pm$ 34.0\%} & {\scriptsize $\pm$ 9.6\%} \\
    \bottomrule
    \end{tabular}
    \caption{Accuracy comparison of different methods showing the mean $\pm$ standard deviation of the relative error computed across all the pairwise distances in the DOTmark class at $128\times128$ resolution. Each method is evaluated at different fidelity level $\kappa_2=2$ and $\kappa_4=4$ and different values of $\varepsilon_1 = 1\cdot 10^{-3}N^p$ and $\varepsilon_4=4\cdot 10^{-3}N^p$. }
    \label{tbl: accuracy}
\end{table}

\begin{figure}
  \centering
  \begin{subfigure}
    \centering
    \includegraphics[width=\linewidth]{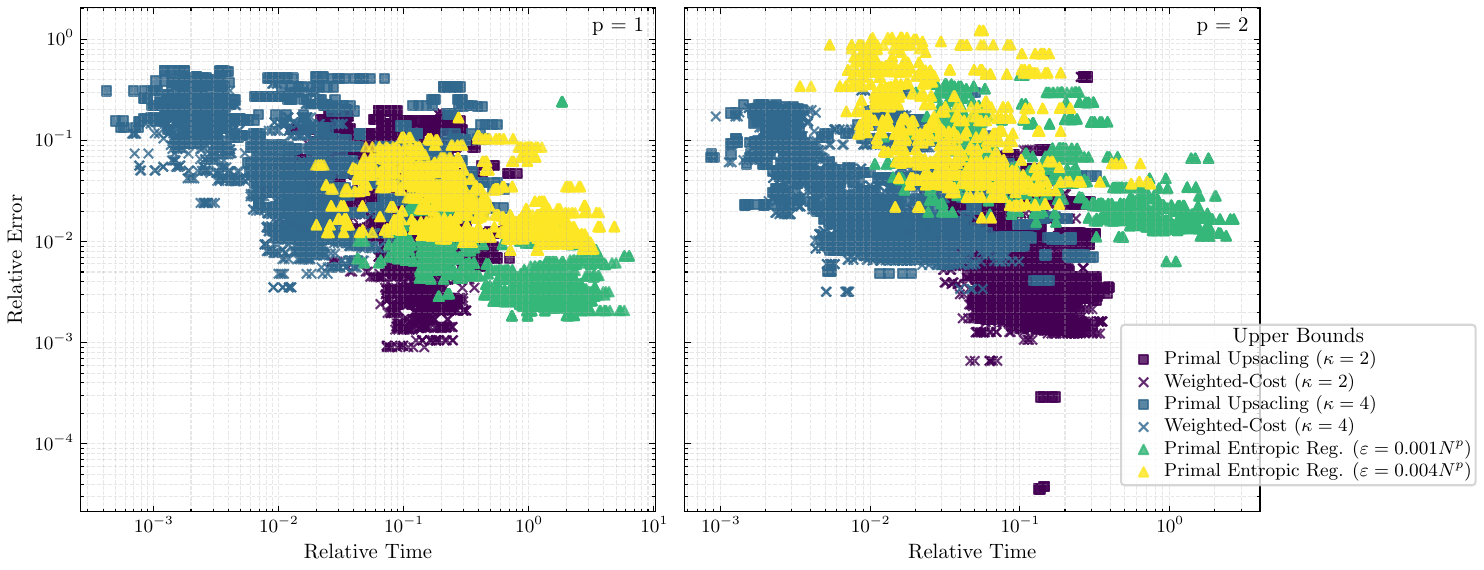}
    \vspace{-2em}
    \caption{Efficiency of Wasserstein Upper Bounds}
    \label{fig: efficiency of wasserstein upper bounds}
  \end{subfigure}
  \vspace{1em}
  \begin{subfigure}
    \centering
    \includegraphics[width=\linewidth]{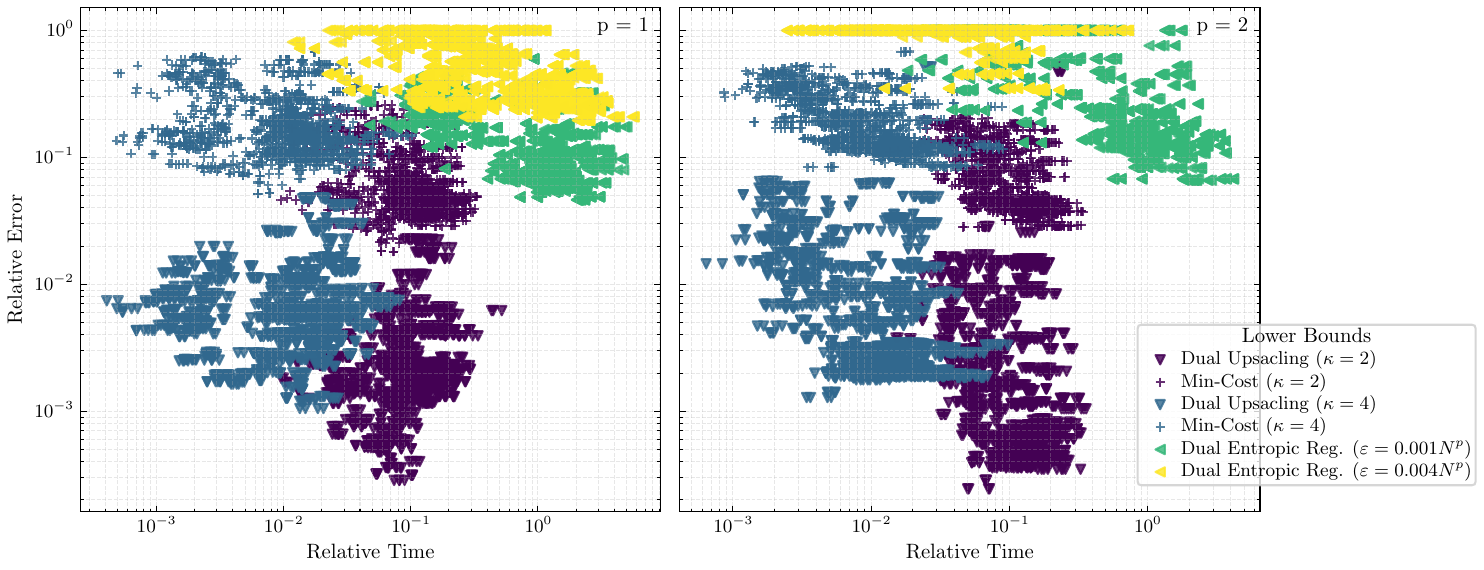}
    \vspace{-2em}
    \caption{Efficiency of Wasserstein Lower Bounds}
    \label{fig: efficiency of wasserstein bounds}
  \end{subfigure}
  \label{fig: efficiency of wasserstein lower bounds}
\end{figure}

\paragraph{EMDB}
In the field of structural biology, approximations of the Wasserstein metric are increasingly being used on 2D projection images and 3D volumetric reconstructions of proteins and other macromolecules.
Specific applications include molecular alignment, clustering and dimensionality reduction,
with most methods substituting the Wasserstein metric with a crude approximation that is fast to compute
 \citep{RaoMoscovichSinger2020,RiahiWoollardPoitevinCondonDuc2023,singerAlignmentDensityMaps2024,KileelMoscovichZeleskoSinger2021}.
 We evaluate our algorithms in a challenging 3D alignment setting, where we wish to compute the Wasserstein-$p$ metric $p \in \{1,2\}$ between rotated 3D density maps of the same molecule.
The volumetric density maps are downloaded from the Electron Microscopy Data Bank (EMDB) \citep{ThewwPDBConsortium2024} using the ASPIRE package
\citep{WrightAndenBansalXiaLangfieldEtAl2025}.
\cref{fig:cryo-em-rotations} shows the computed bounds for rotations between $0^\circ$ and $180^\circ$ of the Plasmodium falciparum 80S ribosome 3D density map \citep{WongBaiBrownFernandezHanssenEtAl2014}.

Plots for other molecules are shown in the supplementary material.
Dual upscaling (lower bound) and weighted-cost (upper bound) methods at $\kappa=2$ provide the best approximations. A summary of the computational speed-up is provided in \cref{tbl: time 3d}.
\begin{figure}[h]
  \hspace{-6.5em}
  \includegraphics[width=0.58\linewidth]{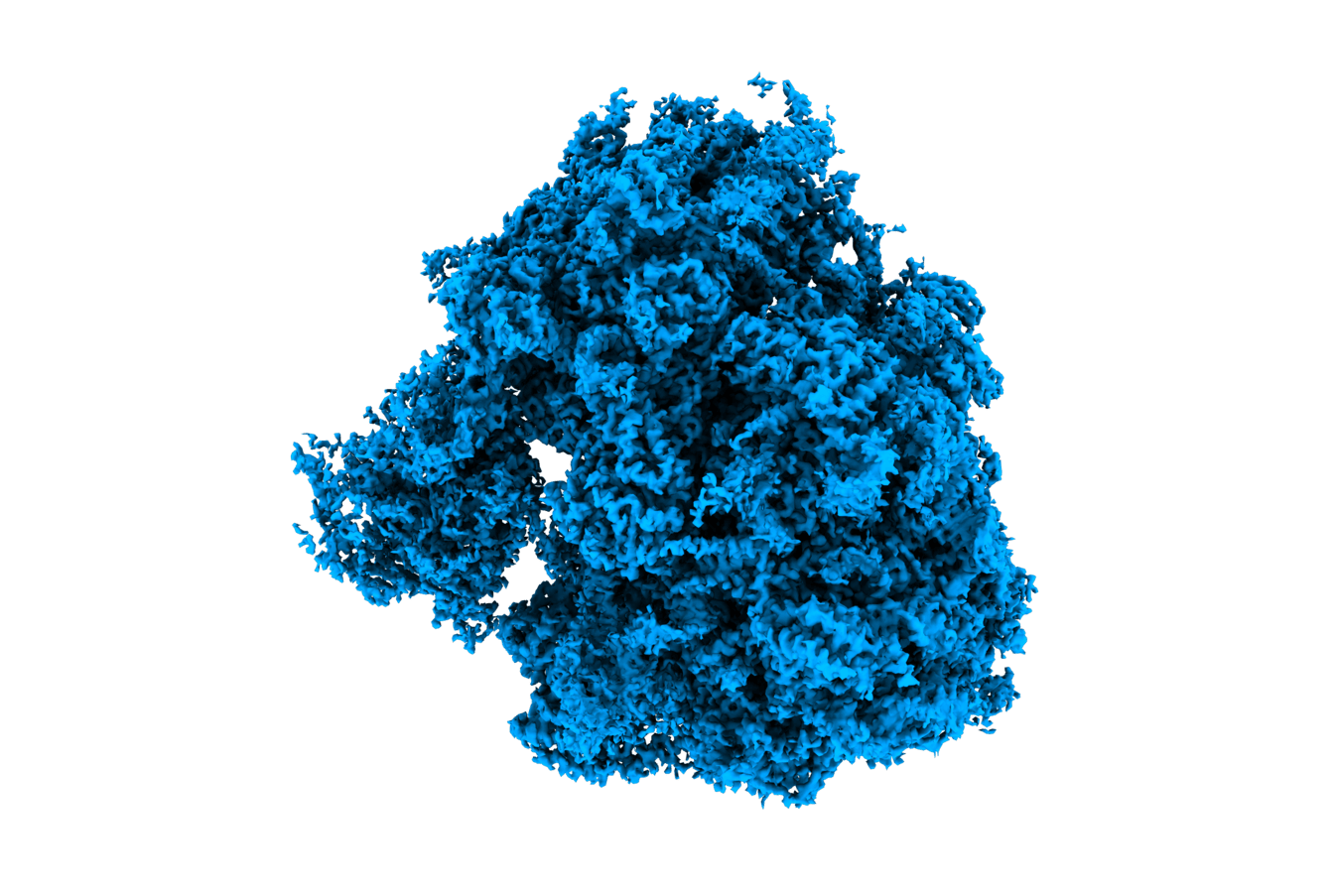}
  \hspace{-5em}
  \includegraphics[width=0.75\linewidth]{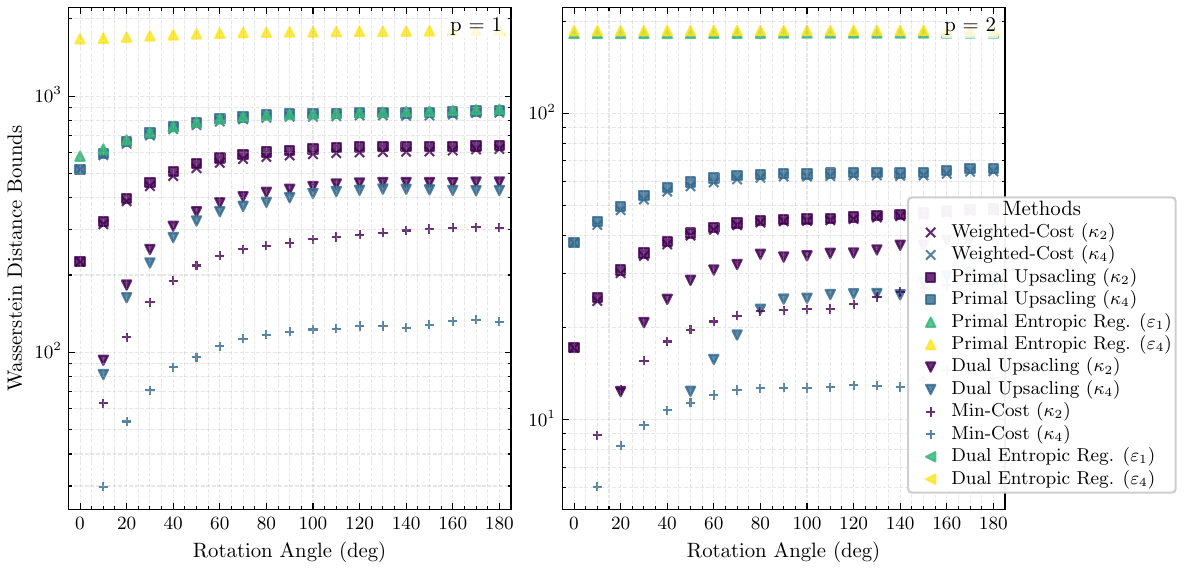}
  \caption{
  Wasserstein distance bounds between rotated 3D density maps of the 80S ribosome.
  The left panel shows an isosurface plot of the 3D density map that we rotated around the $z$-axis.
  The other two panels compare the different algorithms for producing upper and lower bounds on the Wasserstein-$p$ metric. (center) $p=1$; (right) $p=2$.
  }
  \label{fig:cryo-em-rotations}
\end{figure}

\begin{table}[ht]
    \centering
    \small
    \setlength{\tabcolsep}{2pt}
    \begin{tabular}{lcccccc|cccccc}
    \toprule
     & \multicolumn{6}{c}{Upper Bounds} & \multicolumn{6}{c}{Lower Bounds} \\
    \cmidrule{2-7}\cmidrule{8-13}
     & \multicolumn{2}{c}{\shortstack{Weighted-\\Cost}} & \multicolumn{2}{c}{\shortstack{Primal\\Upscaling}} & \multicolumn{2}{c}{\shortstack{Entropic\\Regularization}} & \multicolumn{2}{c}{\shortstack{Dual\\Upscaling}} & \multicolumn{2}{c}{\shortstack{Min-\\Cost}} & \multicolumn{2}{c}{\shortstack{Entropic\\Regularization}} \\
    p & $\kappa_2$ & $\kappa_4$ & $\kappa_2$ & $\kappa_4$ & $\varepsilon_1$ & $\varepsilon_4$ & $\kappa_2$ & $\kappa_4$ & $\kappa_2$ & $\kappa_4$ & $\varepsilon_1$ & $\varepsilon_4$ \\
    \midrule
    \multirow[t]{2}{*}{1} & 0.22\% & \textbf{0.12\%} & 1.09\% & 27.23\% & 130.95\% & 144.97\% & 0.55\% & 0.33\% & 0.20\% & \textbf{0.11\%} & 190.88\% & 197.35\% \\
     & {\scriptsize $\pm$ 0.06\%} & {\scriptsize $\pm$ 0.26\%} & {\scriptsize $\pm$ 0.34\%} & {\scriptsize $\pm$ 15.99\%} & {\scriptsize $\pm$ 29.46\%} & {\scriptsize $\pm$ 39.00\%} & {\scriptsize $\pm$ 0.21\%} & {\scriptsize $\pm$ 0.21\%} & {\scriptsize $\pm$ 0.06\%} & {\scriptsize $\pm$ 0.29\%} & {\scriptsize $\pm$ 58.51\%} & {\scriptsize $\pm$ 52.12\%} \\
    \cline{1-13}\noalign{\vskip 1.5pt}
    \multirow[t]{2}{*}{2} & 0.22\% & \textbf{0.06\%} & 0.30\% & 0.16\% & 164.64\% & 171.42\% & 0.19\% & \textbf{0.04\%} & 0.21\% & \textbf{0.04\%} & 195.49\% & 197.76\% \\
     & {\scriptsize $\pm$ 0.04\%} & {\scriptsize $\pm$ 0.02\%} & {\scriptsize $\pm$ 0.07\%} & {\scriptsize $\pm$ 0.08\%} & {\scriptsize $\pm$ 42.18\%} & {\scriptsize $\pm$ 41.89\%} & {\scriptsize $\pm$ 0.05\%} & {\scriptsize $\pm$ 0.02\%} & {\scriptsize $\pm$ 0.05\%} & {\scriptsize $\pm$ 0.02\%} & {\scriptsize $\pm$ 42.43\%} & {\scriptsize $\pm$ 44.34\%} \\
    \bottomrule
    \end{tabular}
    \caption{Computation time relative to the exact Wasserstein distance computation for 3D Cryo-EM data at $32\times32\times32$ resolution. Results show mean $\pm$ standard deviation across rotations.}\label{tbl: time 3d}
\end{table}

\section{Conclusion and discussion}\label{sec:conclusion}
In this paper, we proposed several methods for computing fast approximations that lower or upper-bound the Wasserstein metric between discrete distributions on a regular grid.
Our methods are based on the solution of lower-resolution OT problems that are then upsampled and corrected to yield upper and lower bounds to the original problem.
Our experiments on 2D images and 3D volumetric data demonstrate significant improvements in
computational efficiency and accuracy compared to bounds based on entropic OT.
Despite the considerable computational speedups resulting from our methods compared to the only apparent alternative for exact Wasserstein bounds, the methods are still much slower compared to many almost linear time state-of-the-art approximation methods practitioners use in in practice for such datasets.
Looking forward, our approach could be refined and extended by exploring different interpolation methods for the upscaling stage and multi-scale approaches.
Finally, despite the paper's focus, the methods could also be extended to domains beyond regular grids such as point clouds in $\mathbb{R}^n$ and graphs.

\FloatBarrier
\bibliography{optimal-transport-approximation}
\bibliographystyle{abbrvnat}

\newpage
\appendix
\section{Algorithms}\label{sec: algorithms}
Lower-bound based on entropic regularization described in \cref{sec: reg lower bound}.
This method simply runs the iterative Sinkhorn algorithm an then returns the unregularized cost term of the dual solution to the entropic regularization problem.
\begin{algorithm}
    \caption{Regularization-Based Lower Bound}
    \label{alg: regularization lower bound}
    \begin{algorithmic}
        \REQUIRE $\vmu,\vnu \in \Sigma_{N^d}$ on $\calX = [N]^d$, $p \ge 1, \varepsilon > 0, \xi > 0$ and metric $\rho$.
        \STATE $C \gets \{\rho(x_i,y_j)^p\}_{ij}$
        \STATE $K \gets \exp\left( -\frac{C}{\varepsilon} \right)$
        \STATE Initialize $\vf \gets \zeros_n$, $\vg \gets \zeros_m$
        \STATE Initialize $\vb \gets \exp\left( \frac{\vg}{\varepsilon} \right)$
        \REPEAT
        \STATE $\va \gets \vmu \oslash (K \vb)$
        \STATE $\vf \gets \varepsilon \log \va$
        \STATE $\vb \gets \vnu \oslash (K^\top \va)$
        \STATE $\vg \gets \varepsilon \log \vb$
        \UNTIL{$\|\va \odot K\vb - \vmu\|_1 + \|\vnu - \vb \odot K^\top \va\|_1 < \xi$}
        \RETURN $\left(\inner{\vf,\vmu} + \inner{\vg,\vnu}\right)^\frac{1}{p}$
    \end{algorithmic}
\end{algorithm}

Algorithm for entropic regularization-based upper bound described in \cref{sec: reg upper bound}, using the unregularized term of the primal solution to the entropic regularization problem, with weighted total variation marginal corrections.
\begin{algorithm}
    \caption{Regularization-Based Upper Bound}
    \label{alg: regularization upper bound}
    \begin{algorithmic}
        \REQUIRE $\vmu,\vnu \in \Sigma_{N^d}$ on $\calX = [N]^d$, $p \ge 1, \varepsilon > 0, \xi > 0$ and metric $\rho$.
        \STATE $C \gets \{\rho(x_i,y_j)^p\}_{ij}$
        \STATE $K \gets \exp\left( -\frac{C}{\varepsilon} \right)$
        \STATE Initialize $\vb \gets \ones_{N^d}$
        \REPEAT
        \STATE $\va \gets \vmu \oslash K\vb$
        \STATE $\vb \gets \vnu \oslash K^\top \va$
        \STATE $\hat{\vmu} \gets \va \odot (K \vb)$
        \STATE $\hat{\vnu} \gets \vb \odot (K^\top \va)$
        \UNTIL{$\|\hat{\vmu} - \vmu\|_1 + \|\vnu - \hat{\vnu}\|_1 < \xi$}
        \STATE $\hat\pi_\varepsilon \gets \diag{(\va)}K\diag{(\vb)}$
        \STATE $\bar{x} \gets \code{mean}(\calX)$
        \STATE $\vw \gets \{\rho(\bar{x},x_i)^p\}_i$
        \STATE $\Delta_{\hat\mu} \gets 2^{1-\frac{1}{p}} \inner{\vw, |\hat{\vmu} - \vmu|}^{\frac{1}{p}}$
        \STATE $\Delta_{\hat\nu} \gets 2^{1-\frac{1}{p}} \inner{\vw, |\vnu - \hat{\vnu}|}^{\frac{1}{p}}$
        \RETURN $\inner{\hat\pi_\varepsilon, C}^\frac{1}{p} + \Delta_{\hat\mu} + \Delta_{\hat\nu}$
    \end{algorithmic}
\end{algorithm}

Algorithm for quantization-based upper bound described in \cref{sec: Weighted-Cost Upper Bound}, using coarse cost weighted by the trivial coupling.
\begin{algorithm}[H]
    \caption{Weighted-Cost Upper Bound}
    \label{alg: weighted cost upper bound}
    \begin{algorithmic}
        \REQUIRE $\vmu,\vnu \in \Sigma_{N^d}$ on $\calX = [N]^d$, $p \ge 1, \kappa\in\N$ and metric $\rho$.
        \STATE $\tilde{\vmu} \gets \code{SumPool}(\vmu;\kappa)$
        \STATE $\tilde{\vnu} \gets \code{SumPool}(\vnu;\kappa)$
        \STATE $\bar{C}_{k\ell}
        \gets
        \left\{\frac{1}{\tilde\mu_k \tilde\nu_\ell} \sum_{\substack{
        x\in X_k \\
        y\in Y_\ell}}
        \rho(x,y)^p \mu(x)\nu(y) \right\}_{k\ell}$
        \STATE Solve $L_{\bar{C}} \gets \min_{\tilde{\pi} \in \Pi(\tilde{\vmu}, \tilde{\vnu})} \inner{\tilde{\pi}, \bar{C}}$
        \RETURN ${L_{\bar{C}} }^{\frac{1}{p}}$
    \end{algorithmic}
\end{algorithm}

Algorithm for bi-level quantization-based upper bound described in \cref{sec: Upscaling upper bound}, using nearest-neighbor upscaling of the optimal coarse coupling, iterative proportional fitting of the marginals (i.e. Sinkhorn iterations), with weighted total variation marginal corrections.
\begin{algorithm}[H]
    \caption{Upscaling Upper Bound}
    \label{alg: upscaling upper bound}
    \begin{algorithmic}
        \REQUIRE $\vmu,\vnu \in \Sigma_{N^d}$ on $\calX=[N]^d$,\, $p \ge 1,\, \kappa\in\N,\, \xi>0$ and metric $\rho$.
        \STATE $\tilde{\calX} \gets \code{AvgPool}(\calX;\kappa)$
        \STATE $\tilde{\vmu} \gets \code{SumPool}(\vmu;\kappa)$
        \STATE $\tilde{\vnu} \gets \code{SumPool}(\vnu;\kappa)$
        \STATE $\tilde{C} \gets \{\rho(\tilde{x}_k, \tilde{x}_\ell)^p \}_{k\ell}$
        \STATE Solve $\tilde{\pi}^* \gets \argmin_{\tilde{\pi} \in \Pi(\tilde{\vmu}, \tilde{\vnu})} \inner{\tilde{\pi}, \tilde{C}}$
        \STATE \COMMENT{\nameref{p: step 1}}
        \STATE $\tilde{\tP}^* \gets \code{reshape}(\tilde{\pi}^*;n)$ \hfill\COMMENT{Reshape as tensor}
        \STATE $\tK \gets \{\kappa^{-2d}\}_{t \in [\kappa]^{2d}}$
        \STATE $\hat{\tP} \gets \tilde{\tP}^* \otimes \tK$ \hfill\COMMENT{Upscaling}
        \STATE $\hat{\pi} \gets \code{reshape}^{-1}(\hat{\tP};N)$
        \STATE \COMMENT{\nameref{p: step 2}}
        \STATE Initialize $\vb \gets \ones_{N^d}$
        \REPEAT
        \STATE $\va \gets \vmu \oslash \hat{\pi}\vb$
        \STATE $\vb \gets \vnu \oslash \hat{\pi}^\top \va$
        \STATE $\hat{\vmu} \gets \va \odot (\hat{\pi} \vb)$
        \STATE $\hat{\vnu} \gets \vb \odot (\hat{\pi}^\top \va)$
        \UNTIL{$\|\hat{\vmu} - \vmu\|_1 + \|\vnu - \hat{\vnu}\|_1 < \xi$}
        \STATE \COMMENT{\nameref{p: step 3}}
        \STATE $\widehat{\Wass}_p \gets \left(\sum_{(i, j) \in \supp(\hat{\pi}) } a_i\hat{\pi}_{ij}b_j \, \rho(x_i, x_j)^p \right)^{\frac{1}{p}}$
        \STATE $\bar{x} \gets \code{mean}(\calX)$
        \STATE $\vw \gets \{\rho(\bar{x},x_i)^p\}_i$
        \STATE $\Delta_{\hat\mu} \gets 2^{1-\frac{1}{p}} \inner{\vw, |\hat{\vmu} - \vmu|}^{\frac{1}{p}}$
        \STATE $\Delta_{\hat\nu} \gets 2^{1-\frac{1}{p}} \inner{\vw, |\vnu - \hat{\vnu}|}^{\frac{1}{p}}$
        \RETURN $\widehat{\Wass}_p + \Delta_{\hat\mu} + \Delta_{\hat\nu}$
    \end{algorithmic}
\end{algorithm}

Algorithm for bi-level quantization-based lower bound described in \cref{sec: Upscaling Lower Bound}, using interpolation for upscaling the optimal coarse dual potentials, and c-transform to achieve optimized admissible dual potentials pair.
\begin{algorithm}[H]
    \caption{Upscaling Lower Bound}
    \label{alg: upscaling lower bound}
    \begin{algorithmic}
        \REQUIRE $\vmu,\vnu \in \Sigma_{N^d}$ on $\calX = [N]^d$, $p \ge 1, \kappa\in\N$ and metric $\rho$.
        \STATE $\tilde{\calX} \gets \texttt{AvgPool}(\calX;\kappa)$
        \STATE $\tilde{\vmu} \gets \code{SumPool}(\vmu;\kappa)$
        \STATE $\tilde{\vnu} \gets \code{SumPool}(\vnu;\kappa)$
        \STATE $\tilde{C} \gets \{\rho(\tilde{x}_k, \tilde{x}_\ell)^p \}_{k\ell}$
        \STATE Solve $(\tilde{\vf}^*,\tilde{\vg}^*) \gets \argmax\limits_{(\tilde{\vf},\tilde{\vg}) \in \mathcal{R}(\tilde{C})} \inner{\tilde{\vf},\tilde\vmu} + \inner{\tilde{\vg}, \tilde{\vnu}}$
        \STATE $\hat{\vf} \gets \{R_{\tilde{\vf}, \tilde{\calX}}(x_i)\}_{i\in[N^d]}$
        \STATE $\vg \gets \left\{\min_{i} \rho(x_i,x_j)^p - \hat{f}_i\right\}_j$
        \STATE $\vf \gets \left\{\min_{j} \rho(x_i,x_j)^p - g_j\right\}_i$
        \RETURN $\big(\inner{\vf,\vmu} + \inner{\vg,\vnu} \big)^{\frac{1}{p}}$
    \end{algorithmic}
\end{algorithm}

\FloatBarrier
\begin{table}
    \centering
    \small
    \setlength{\tabcolsep}{2pt}
    \begin{tabular}{lccccccc|cccccc}
    \toprule
     & & \multicolumn{6}{c}{Upper Bounds} & \multicolumn{6}{c}{Lower Bounds} \\
    \cmidrule{3-8}\cmidrule{9-14}
     & & \multicolumn{2}{c}{\shortstack{Weighted-\\Cost}} & \multicolumn{2}{c}{\shortstack{Primal\\Upscaling}} & \multicolumn{2}{c}{\shortstack{Entropic\\Regularization}} & \multicolumn{2}{c}{\shortstack{Dual\\Upscaling}} & \multicolumn{2}{c}{\shortstack{Min-\\Cost}} & \multicolumn{2}{c}{\shortstack{Entropic\\Regularization}} \\
     & & $\kappa_2$ & $\kappa_4$ & $\kappa_2$ & $\kappa_4$ & $\varepsilon_1$ & $\varepsilon_4$ & $\kappa_2$ & $\kappa_4$ & $\kappa_2$ & $\kappa_4$ & $\varepsilon_1$ & $\varepsilon_4$ \\
    Class & p & \multicolumn{12}{c}{} \\
    \midrule
    \multirow[t]{4}{*}{Classic} & \multirow[t]{2}{*}{1} & 4.4\% & \textbf{0.3\%} & 5.0\% & \textbf{0.3\%} & 19.8\% & 16.0\% & 4.9\% & \textbf{0.2\%} & 4.7\% & 0.3\% & 17.6\% & 18.7\% \\Images
     &  & {\scriptsize $\pm$ 2.1\%} & {\scriptsize $\pm$ 0.2\%} & {\scriptsize $\pm$ 2.6\%} & {\scriptsize $\pm$ 0.4\%} & {\scriptsize $\pm$ 10.1\%} & {\scriptsize $\pm$ 11.2\%} & {\scriptsize $\pm$ 2.6\%} & {\scriptsize $\pm$ 0.1\%} & {\scriptsize $\pm$ 2.3\%} & {\scriptsize $\pm$ 0.1\%} & {\scriptsize $\pm$ 7.8\%} & {\scriptsize $\pm$ 13.3\%} \\
    \cline{2-14}\noalign{\vskip 1.5pt}
     & \multirow[t]{2}{*}{2} & 6.1\% & 0.4\% & 6.1\% & \textbf{0.3\%} & 7.9\% & 1.3\% & 6.1\% & \textbf{0.3\%} & 6.1\% & \textbf{0.3\%} & 12.3\% & 1.3\% \\
     &  & {\scriptsize $\pm$ 3.1\%} & {\scriptsize $\pm$ 0.1\%} & {\scriptsize $\pm$ 3.0\%} & {\scriptsize $\pm$ 0.1\%} & {\scriptsize $\pm$ 6.6\%} & {\scriptsize $\pm$ 0.5\%} & {\scriptsize $\pm$ 3.0\%} & {\scriptsize $\pm$ 0.1\%} & {\scriptsize $\pm$ 2.9\%} & {\scriptsize $\pm$ 0.1\%} & {\scriptsize $\pm$ 10.4\%} & {\scriptsize $\pm$ 0.6\%} \\
    \midrule
    \multirow[t]{4}{*}{Micro-} & \multirow[t]{2}{*}{1} & 16.7\% & \textbf{2.3\%} & 20.7\% & 6.6\% & 128.0\% & 90.7\% & 14.4\% & 2.2\% & 15.7\% & \textbf{1.9\%} & 120.8\% & 86.9\% \\scopy
     &  & {\scriptsize $\pm$ 5.6\%} & {\scriptsize $\pm$ 1.4\%} & {\scriptsize $\pm$ 9.6\%} & {\scriptsize $\pm$ 10.6\%} & {\scriptsize $\pm$ 68.3\%} & {\scriptsize $\pm$ 79.2\%} & {\scriptsize $\pm$ 5.0\%} & {\scriptsize $\pm$ 1.2\%} & {\scriptsize $\pm$ 4.8\%} & {\scriptsize $\pm$ 1.0\%} & {\scriptsize $\pm$ 57.2\%} & {\scriptsize $\pm$ 81.3\%} \\
    \cline{2-14}\noalign{\vskip 1.5pt}
     & \multirow[t]{2}{*}{2} & 15.4\% & \textbf{1.9\%} & 17.4\% & 3.0\% & 75.6\% & 7.7\% & 15.6\% & 2.0\% & 14.9\% & \textbf{1.7\%} & 99.9\% & 8.1\% \\
     &  & {\scriptsize $\pm$ 5.2\%} & {\scriptsize $\pm$ 1.2\%} & {\scriptsize $\pm$ 5.9\%} & {\scriptsize $\pm$ 3.4\%} & {\scriptsize $\pm$ 51.3\%} & {\scriptsize $\pm$ 4.9\%} & {\scriptsize $\pm$ 5.1\%} & {\scriptsize $\pm$ 1.3\%} & {\scriptsize $\pm$ 5.0\%} & {\scriptsize $\pm$ 1.0\%} & {\scriptsize $\pm$ 72.8\%} & {\scriptsize $\pm$ 5.5\%} \\
    \midrule
    \multirow[t]{4}{*}{Shapes} & \multirow[t]{2}{*}{1} & 10.2\% & \textbf{1.9\%} & 15.1\% & 10.8\% & 172.9\% & 68.8\% & 7.8\% & 1.5\% & 8.5\% & \textbf{1.4\%} & 163.6\% & 77.0\% \\
     &  & {\scriptsize $\pm$ 3.6\%} & {\scriptsize $\pm$ 1.7\%} & {\scriptsize $\pm$ 8.9\%} & {\scriptsize $\pm$ 11.9\%} & {\scriptsize $\pm$ 99.0\%} & {\scriptsize $\pm$ 72.7\%} & {\scriptsize $\pm$ 2.0\%} & {\scriptsize $\pm$ 0.7\%} & {\scriptsize $\pm$ 2.4\%} & {\scriptsize $\pm$ 1.2\%} & {\scriptsize $\pm$ 88.8\%} & {\scriptsize $\pm$ 77.6\%} \\
    \cline{2-14}\noalign{\vskip 1.5pt}
     & \multirow[t]{2}{*}{2} & 7.9\% & \textbf{1.5\%} & 9.6\% & 4.2\% & 39.5\% & 6.0\% & 7.4\% & 1.2\% & 7.3\% & \textbf{1.0\%} & 56.2\% & 5.7\% \\
     &  & {\scriptsize $\pm$ 3.8\%} & {\scriptsize $\pm$ 1.8\%} & {\scriptsize $\pm$ 5.6\%} & {\scriptsize $\pm$ 6.4\%} & {\scriptsize $\pm$ 45.0\%} & {\scriptsize $\pm$ 11.6\%} & {\scriptsize $\pm$ 2.6\%} & {\scriptsize $\pm$ 0.6\%} & {\scriptsize $\pm$ 3.0\%} & {\scriptsize $\pm$ 1.4\%} & {\scriptsize $\pm$ 69.1\%} & {\scriptsize $\pm$ 11.8\%} \\
    \bottomrule
    \end{tabular}
    \caption{Computational time comparison of different methods showing the mean $\pm$ standard deviation of the relative computation time compared to exact OT solver. Each method is evaluated at different fidelity level $\kappa_2=2$ and $\kappa_4=4$ and different values of $\varepsilon_1 = 1\cdot 10^{-3}N^p$ and $\varepsilon_4=4\cdot 10^{-3}N^p$.}
    \label{tab:DOTmark time comparison}
\end{table}
\begin{figure}
    \centering
    \begin{subfigure}
      \centering
      \includegraphics[width=\linewidth]{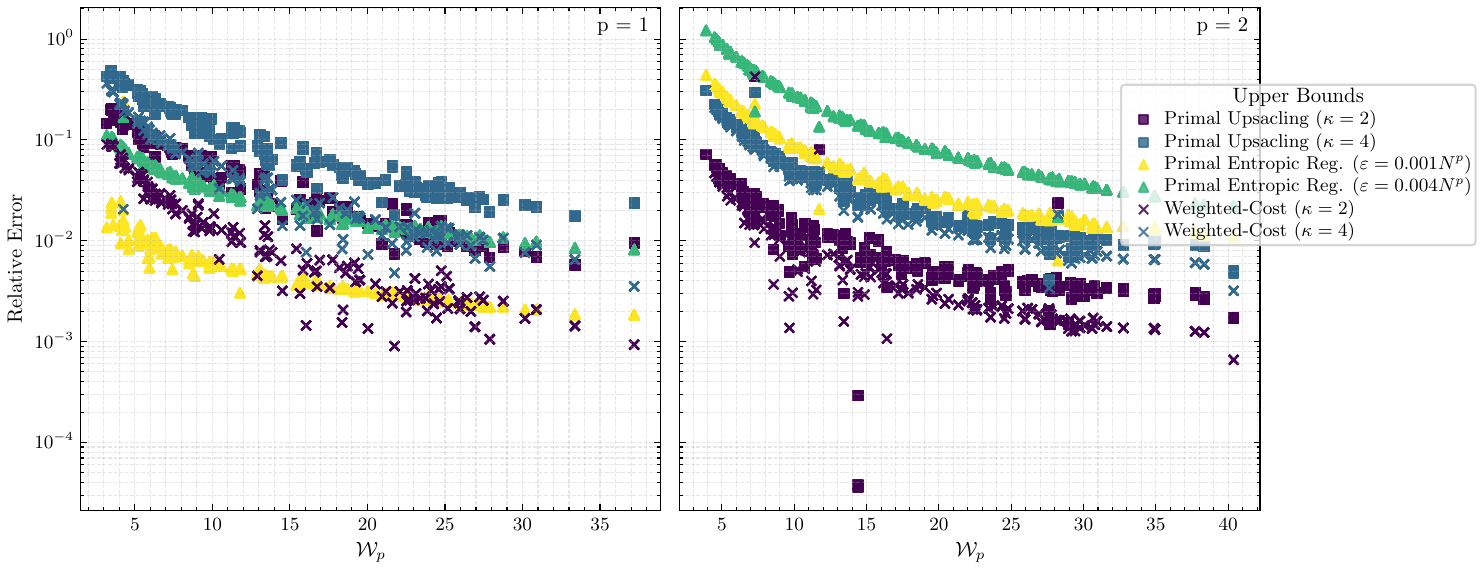}
      \vspace{-2em}
      \caption{Accuracy of Wasserstein Upper Bounds}
      \label{fig: accuracy of wasserstein upper bounds}
    \end{subfigure}
    \begin{subfigure}
      \centering
      \includegraphics[width=\linewidth]{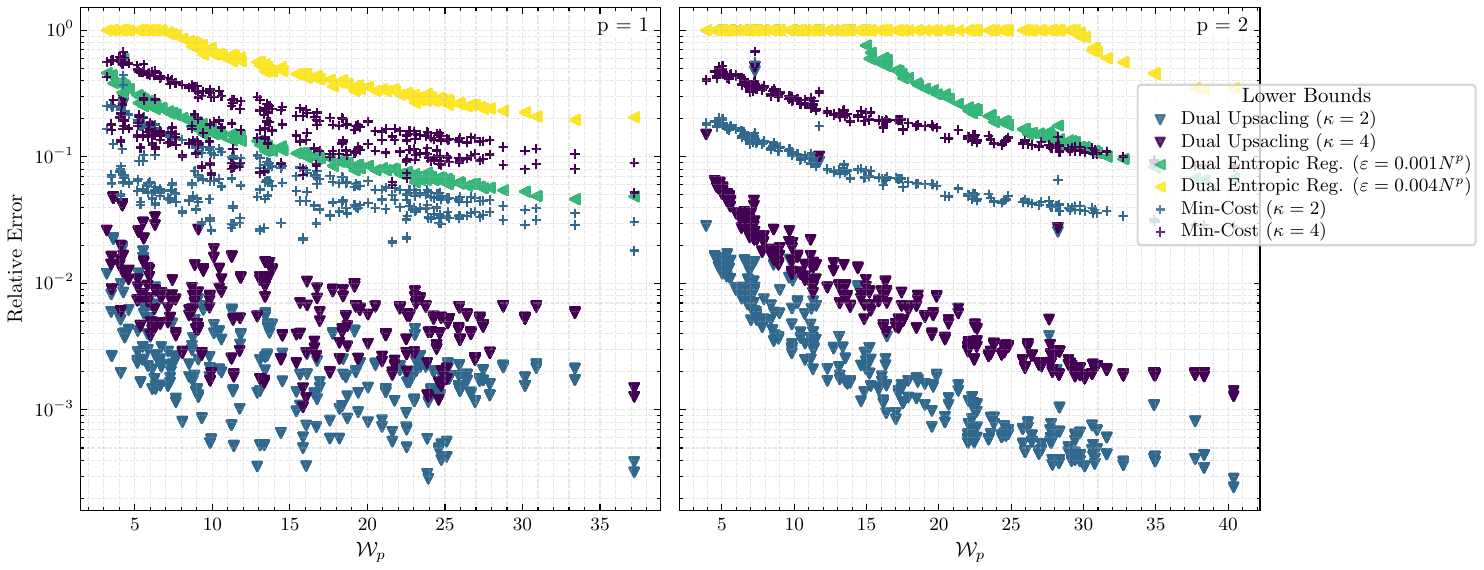}
      \vspace{-2em}
      \caption{Accuracy of Wasserstein Lower Bounds. Negative-valued bounds are clipped to 0, evaluating as 100\% relative error.}
      \label{fig: accuracy of wasserstein lower bounds}
    \end{subfigure}
\end{figure}
\section{Experiments}\label{sec: additional experiments}
\subsection{DOTmark}
In this section we present additional figures and results evaluating the proposed Wasserstein bounds on the discrete optimal transport benchmark (DOTMark) \citep{schrieberDOTmarkBenchmarkDiscrete2017} presented in the main text. The computational speed up of the proposed methods compared to the exact OT solver are summarized in \cref{tab:DOTmark time comparison}. The results show that the quantization methods achieve the most significant speed ups. Notably, the dual upscaling method at $\kappa=4$ are calculated in 0.2-2.2\% of the time, while making almost no sacrifice in accuracy. Maintaining no more than 2.4\% average error.

The relative accuracy of the proposed methods exponentially improves for large values of the exact Wasserstein distance as evident in \cref{fig: accuracy of wasserstein upper bounds,fig: accuracy of wasserstein lower bounds}. Negative-valued lower bounds are trivially clipped to 0, when evaluate in the benchmark.

\subsection{EMDB}
The Electron Microscopy Data Bank (EMDB) \citep{ThewwPDBConsortium2024} is a repository of volumetric density maps that contains many interesting molecules that were reconstructed from cryogenic electron microscopy (cryo-EM) experiments.
These reconstructions are estimates of the 3D electric potential at every point in the molecule.
For our 3D experiments, we downloaded and processed four maps of famous molecules, detailed in \cref{tab:emdb} using the ASPIRE package \citep{WrightAndenBansalXiaLangfieldEtAl2025}.
In \cref{fig:cryo-em-rotations} you can see 3D renderings of these molecules that we generated using UCSF ChimeraX \citep{MengGoddardPettersenCouchPearsonEtAl2023}.
The volumetric maps were downloaded from EMDB, masked inside a spherical region of radius $128$ pixels, rotated around the Z axis in increments of 20 degrees and downscaled to $16 \times 16 \times 16$ voxels.
The computational speed up is summarized in \cref{tbl:relative_time_results}, showing that even at $\kappa=2$ the quantization methods provide substantial speedups.
\cref{fig:emdb-rotations} shows the Wasserstein metrics and bounds between the 3D density map of the molecule in its base orientation and its rotations around the Z axis.
The exact Wasserstein metric is shown as the thick black line with upper and lower bounds next to it using the various markers.
\begin{table}[b]
\caption{Selected cryo-EM structures from the Electron Microscopy Data Bank (EMDB).}
\label{tab:emdb}
\centering
\begin{tabular}{@{}llp{9cm}@{}}
\toprule
\textbf{Name} & \textbf{EMDB ID} & \textbf{Description} \\ \midrule
Ribosome & \href{https://www.ebi.ac.uk/emdb/EMD-2660}{EMD-2660} &
Ribosome of the Plasmodium falciparum parasite which causes malaria in humans \citep{WongBaiBrownFernandezHanssenEtAl2014}\\[2pt]

SARS-CoV-2 & \href{https://www.ebi.ac.uk/emdb/EMD-14621}{EMD-14621} &
SARS-CoV-2 spike protein \citep{StagnoliPeccatiConnellMartinez-CastilloCharroEtAl2022}\\[2pt]

Yeast & \href{https://www.ebi.ac.uk/emdb/EMD-8012}{EMD-8012} &
Yeast spliceosome \citep{NguyenGalejBaiOubridgeNewmanEtAl2016}\\[2pt]

HIV & \href{https://www.ebi.ac.uk/emdb/EMD-2484}{EMD-2484} &
HIV-1 trimeric spike pre-fusion \citep{BartesaghiMerkBorgniaMilneSubramaniam2013}\\ \bottomrule
\end{tabular}
\end{table}
\begin{figure}
    \setlength{\tabcolsep}{0pt}
    \hspace{-1cm}
    \begin{tabular}{cccc}
        & \textbf{XY plane} & \textbf{XZ plane} & \textbf{YZ plane} \\[0.5em]
        \rotatebox{90}{\qquad\quad \bf Ribosome}&\adjustbox{trim=1.5cm 0.5cm 1.5cm 0cm}{
        \includegraphics[height=5.5cm]{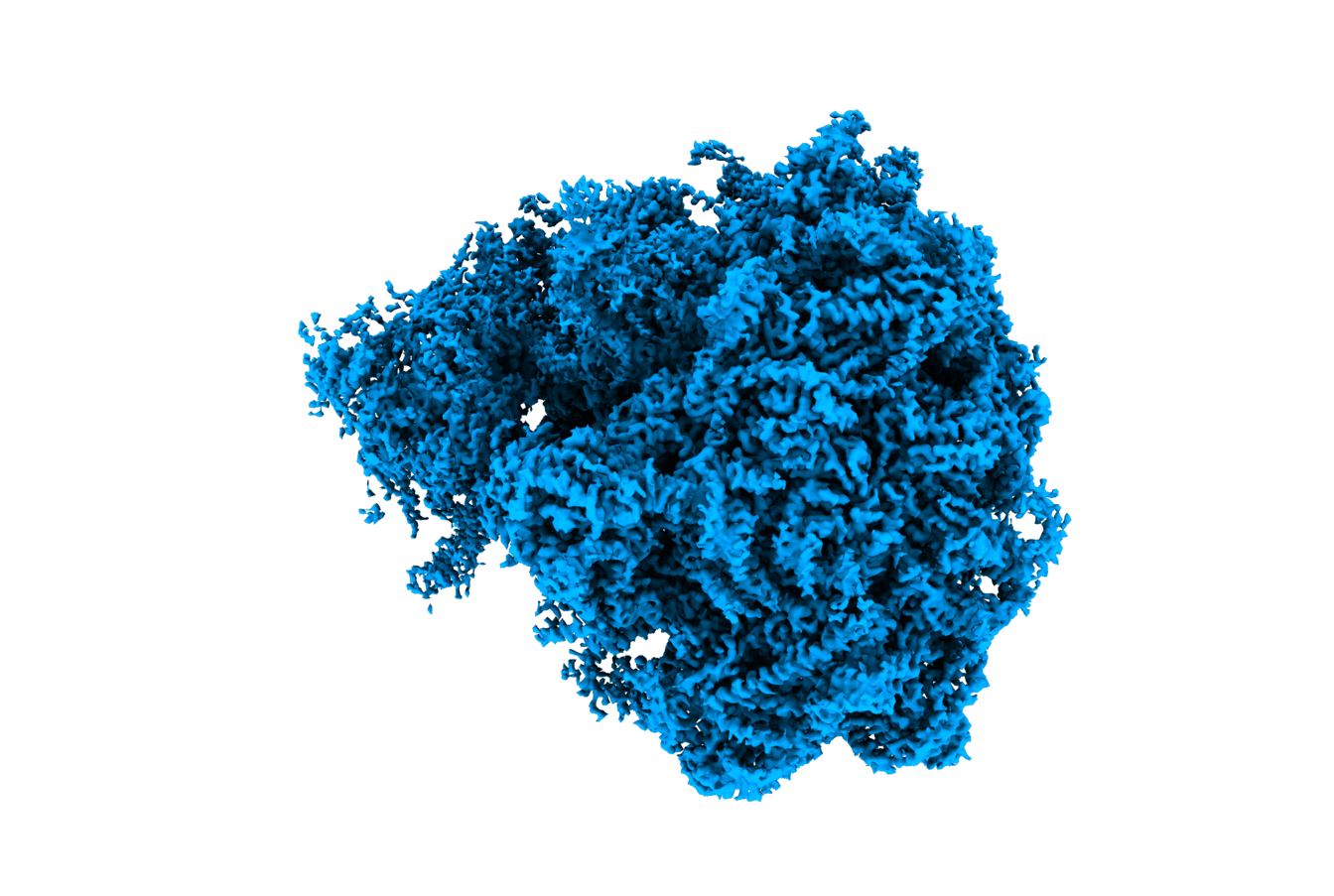}}&\adjustbox{trim=1.9cm 1.0cm 1.5cm 0cm, raise=0.3cm}{\includegraphics[height=5.5cm]{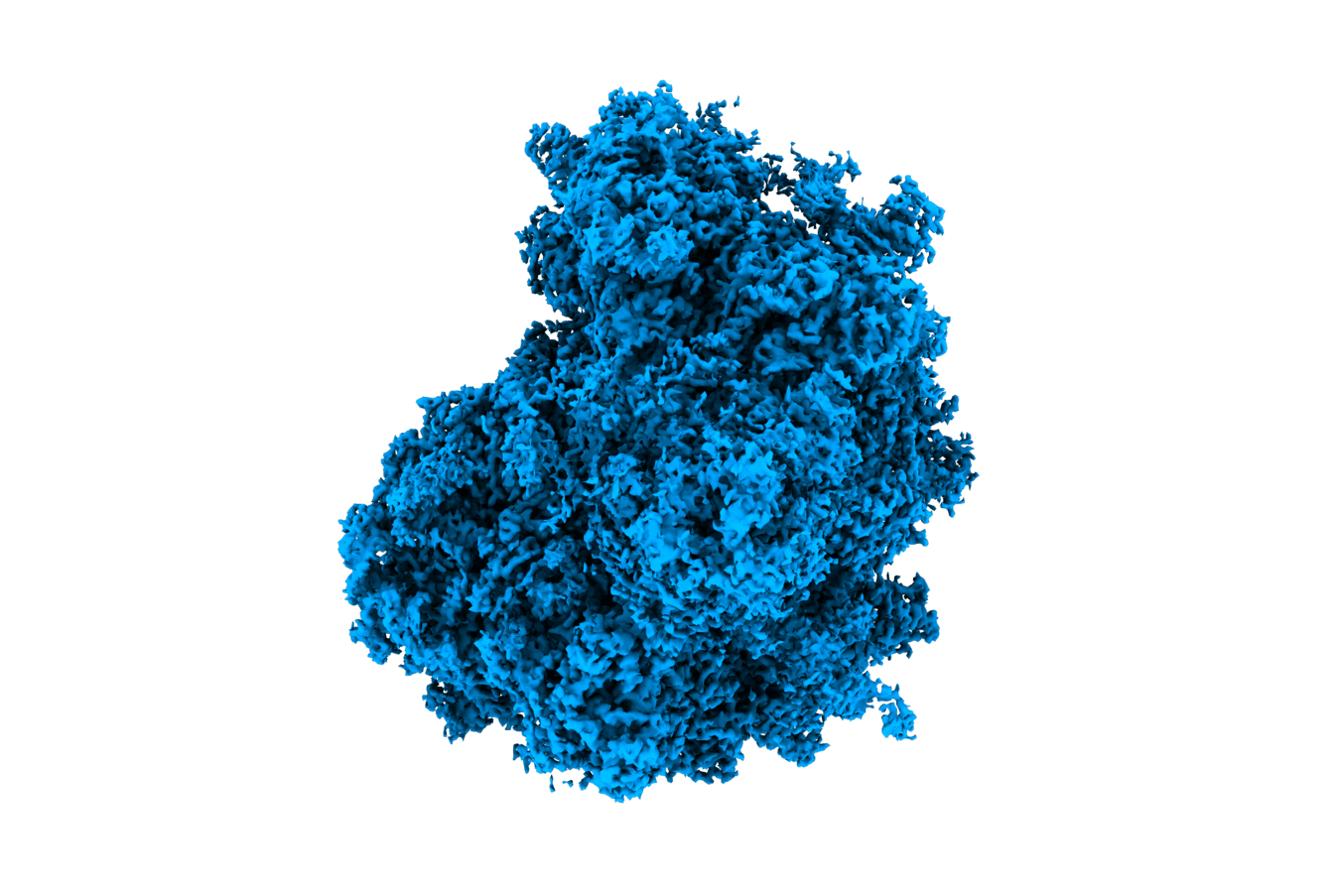}}
        \\
        \rotatebox{90}{\qquad\quad\ \  \bf SARS-CoV-2}&\adjustbox{trim=0.6cm 0.2cm 1.5cm 0cm}{
         \includegraphics[height=5cm]{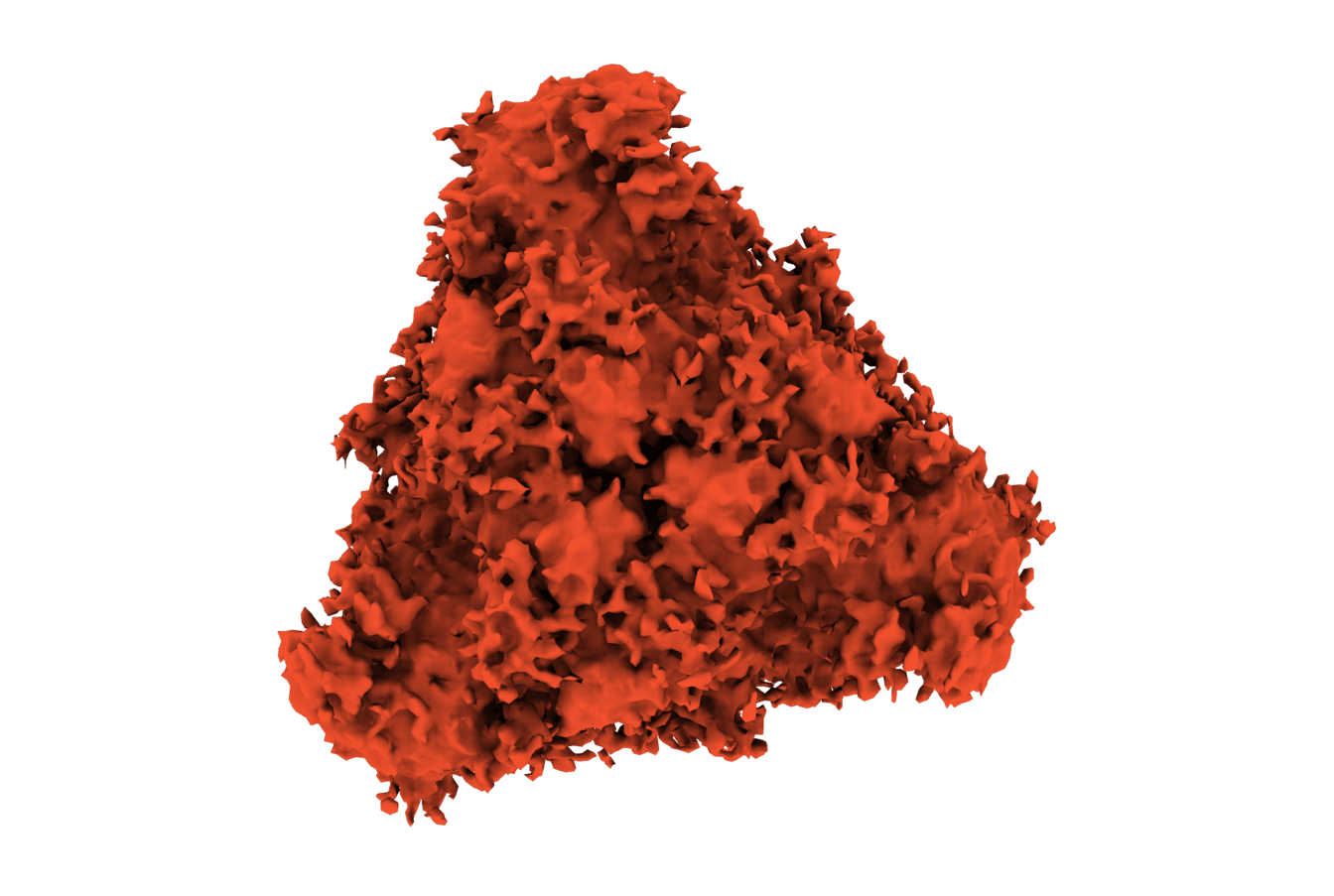}}&\adjustbox{trim=1.5cm 0.2cm 1.5cm 0cm}{
         \includegraphics[height=5cm]{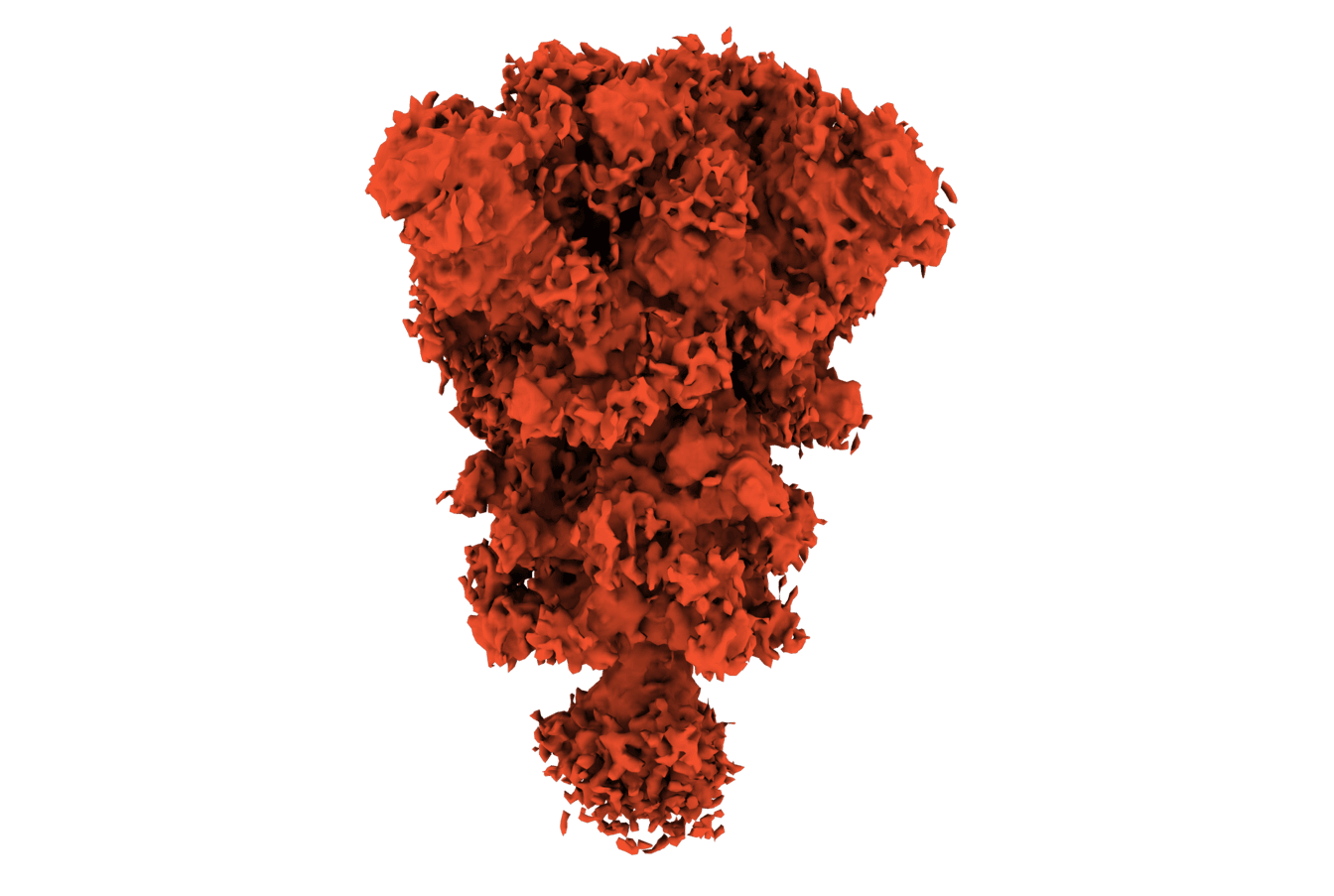}}&\adjustbox{trim=3cm 0.7cm 1.5cm 0cm, raise=0.3cm}{\includegraphics[height=5cm]{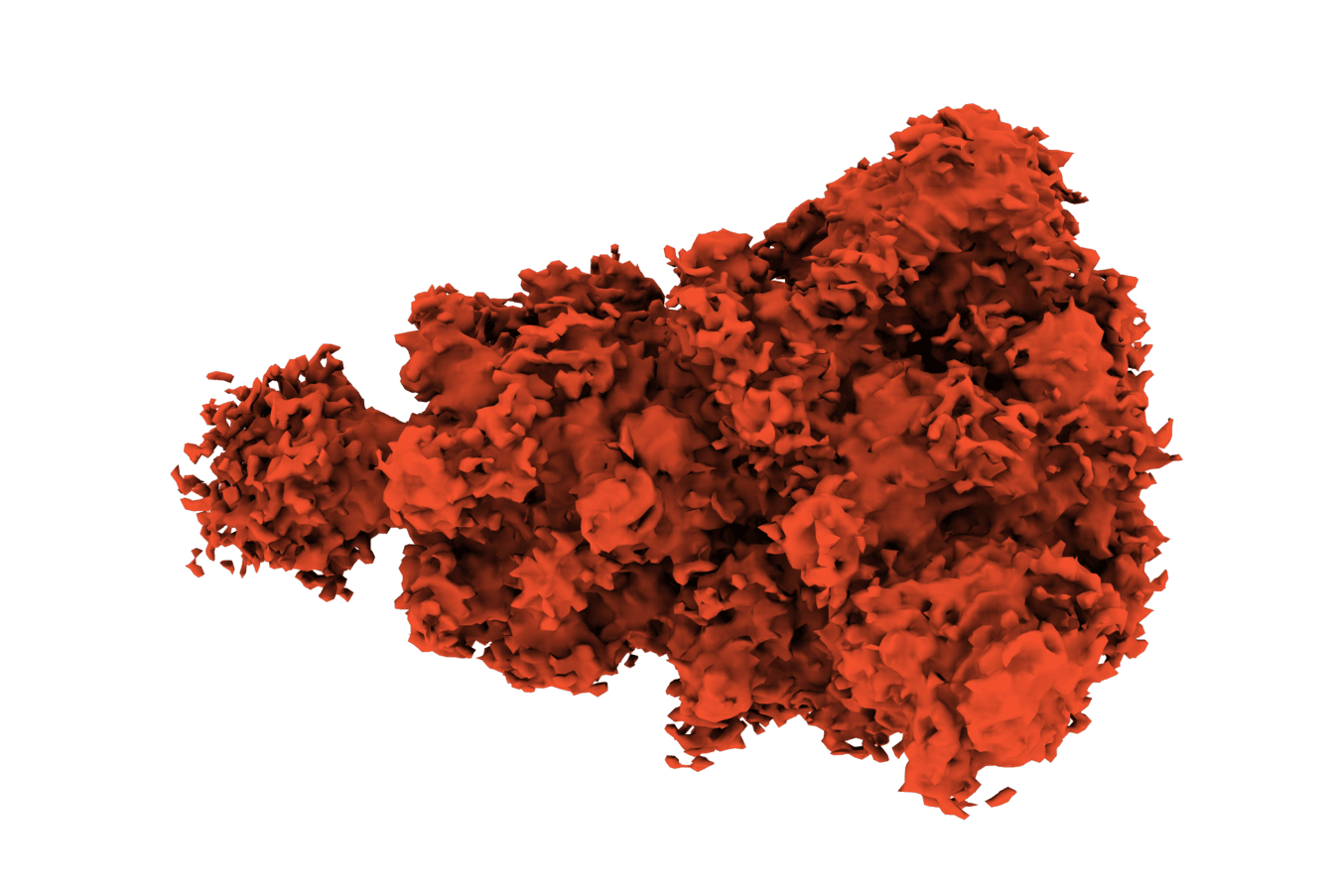}}
        \\
        \rotatebox{90}{\qquad\qquad\qquad \bf Yeast}&\adjustbox{trim=0.2cm 0cm 1.2cm 0.5cm}{\includegraphics[height=5cm]{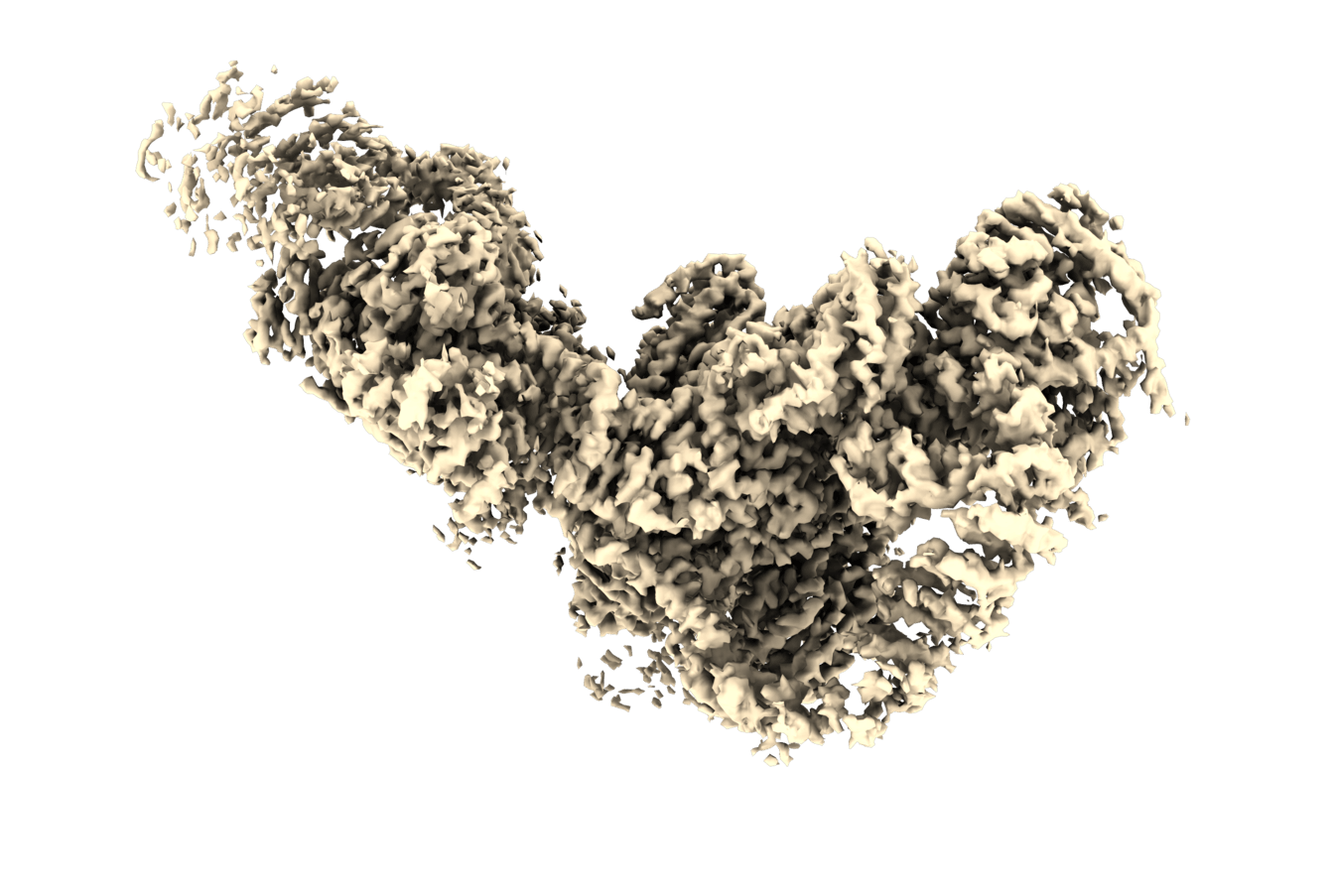}}&\adjustbox{trim=0.8cm 0cm 1.1cm 0.5cm}{\includegraphics[height=5cm]{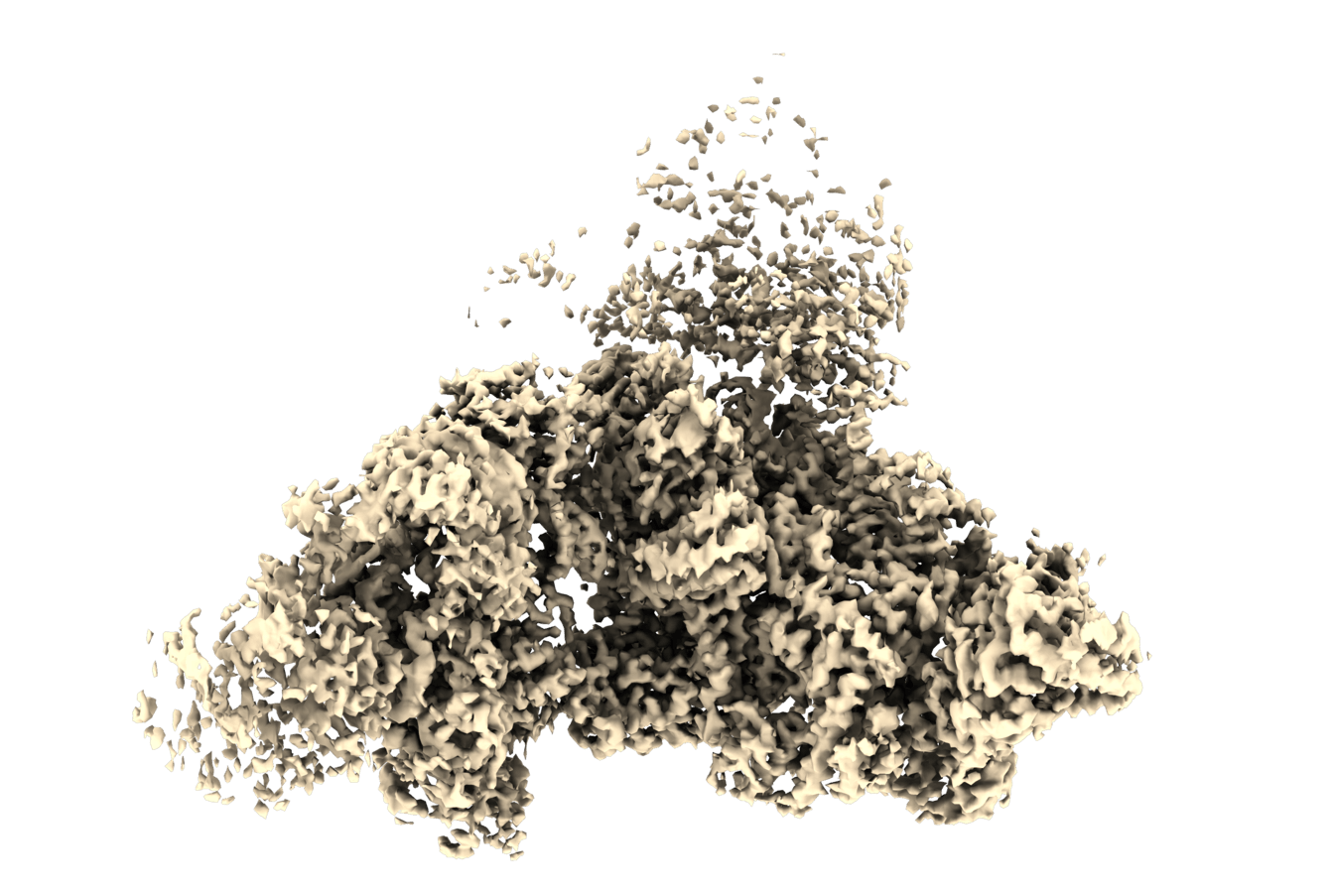}}&\adjustbox{trim=4cm 0.5cm 3cm 0.5cm}{\includegraphics[height=5cm]{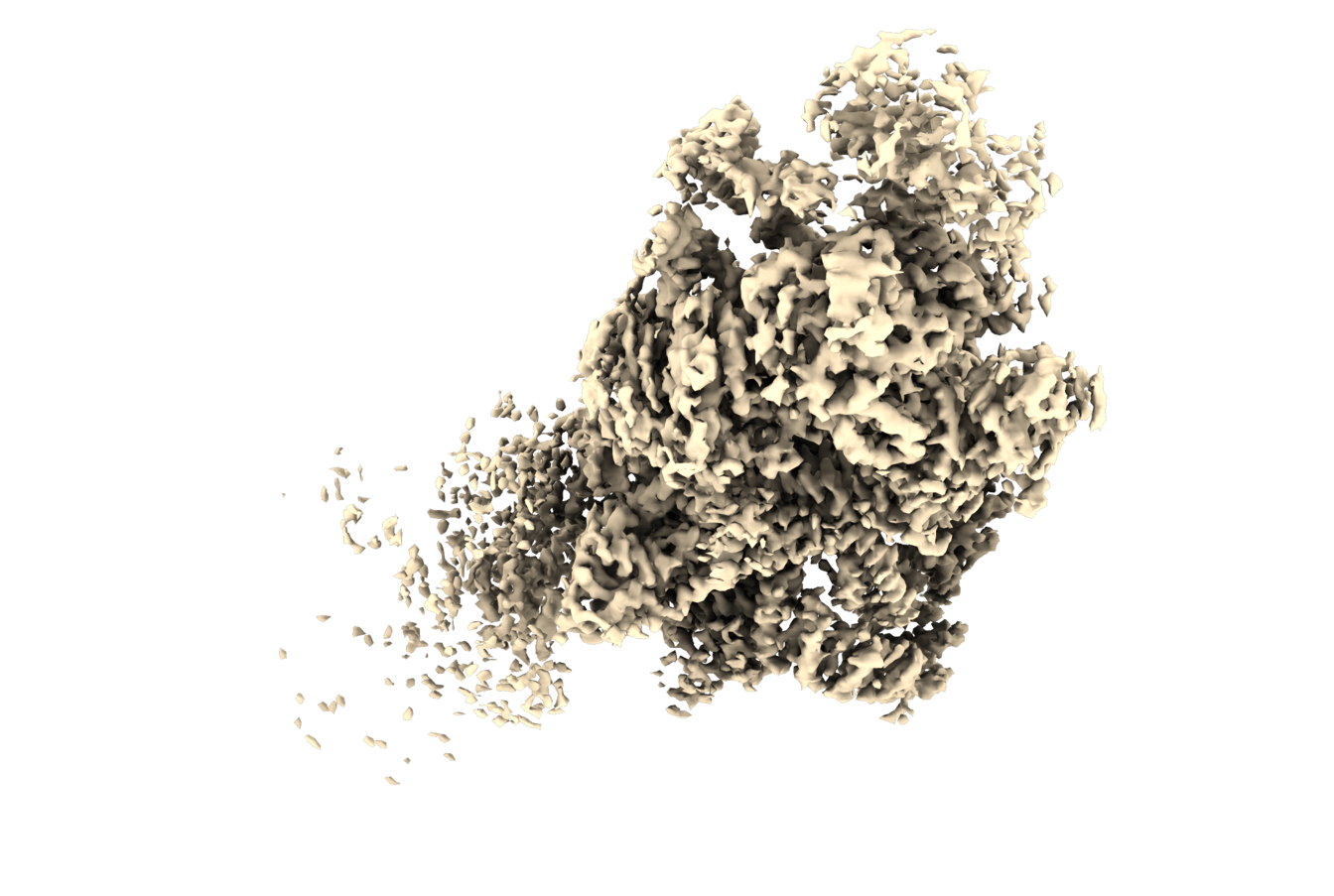}}
        \\
        \rotatebox{90}{\qquad\qquad\quad\  \bf HIV}&\adjustbox{trim=1.2cm 0cm 1.0cm 0.5cm}{
        \includegraphics[height=4.5cm]{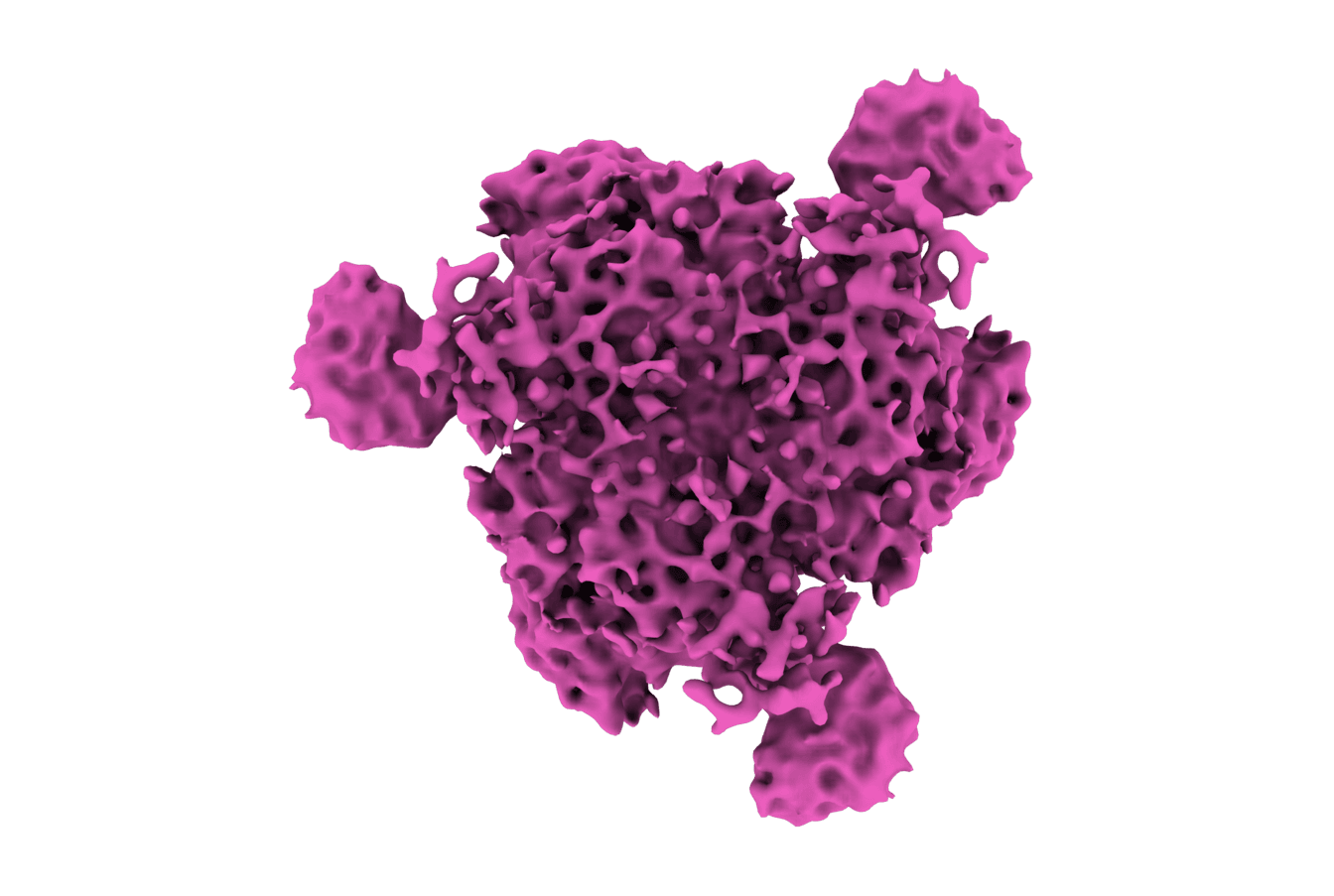}}&\adjustbox{trim=1.8cm 0cm 1.5cm 0.5cm}{\includegraphics[height=4.5cm]{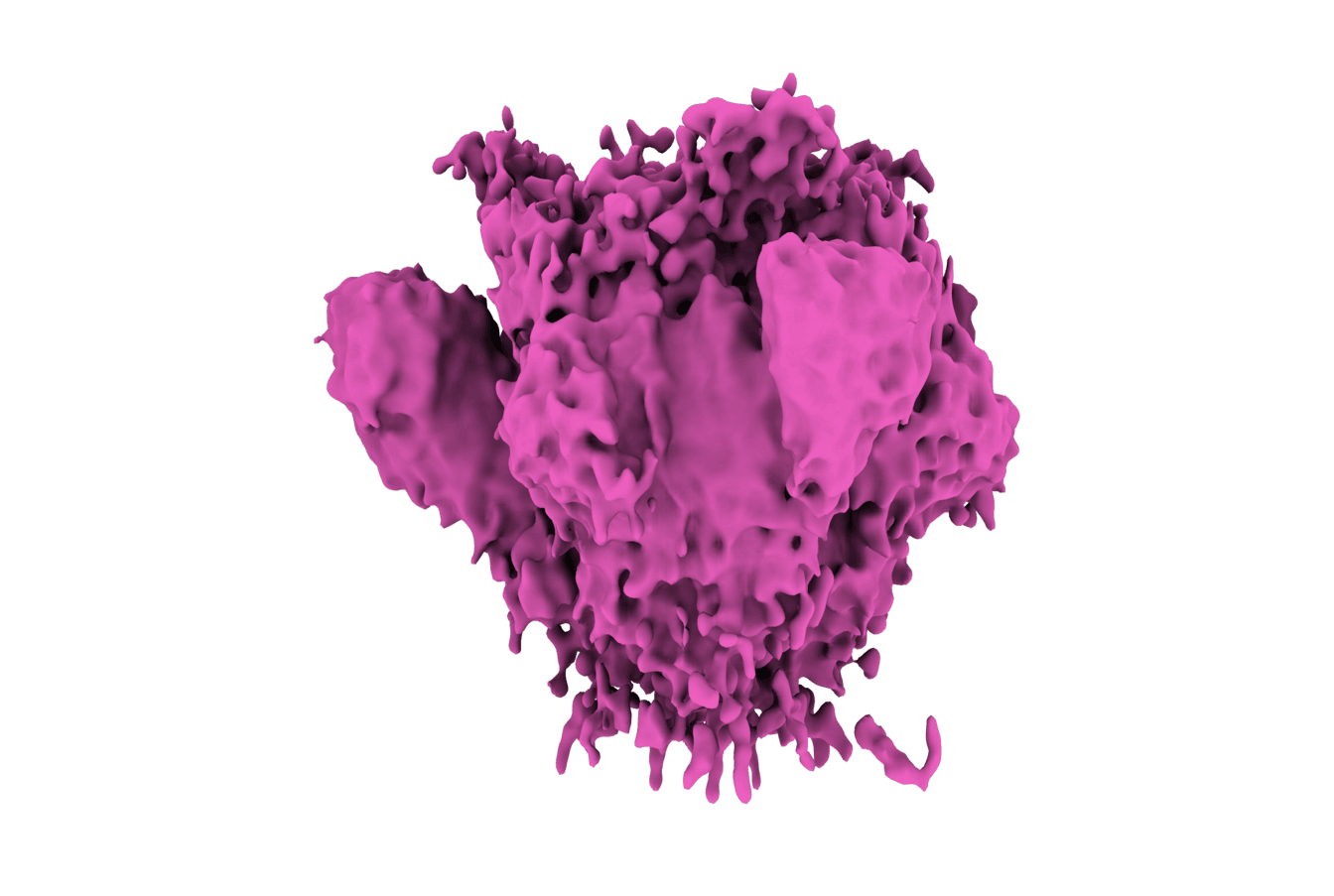}}&\adjustbox{trim=1.0cm 0.5cm 0cm 0.5cm, raise=0.3cm}{
        \includegraphics[height=4.5cm]{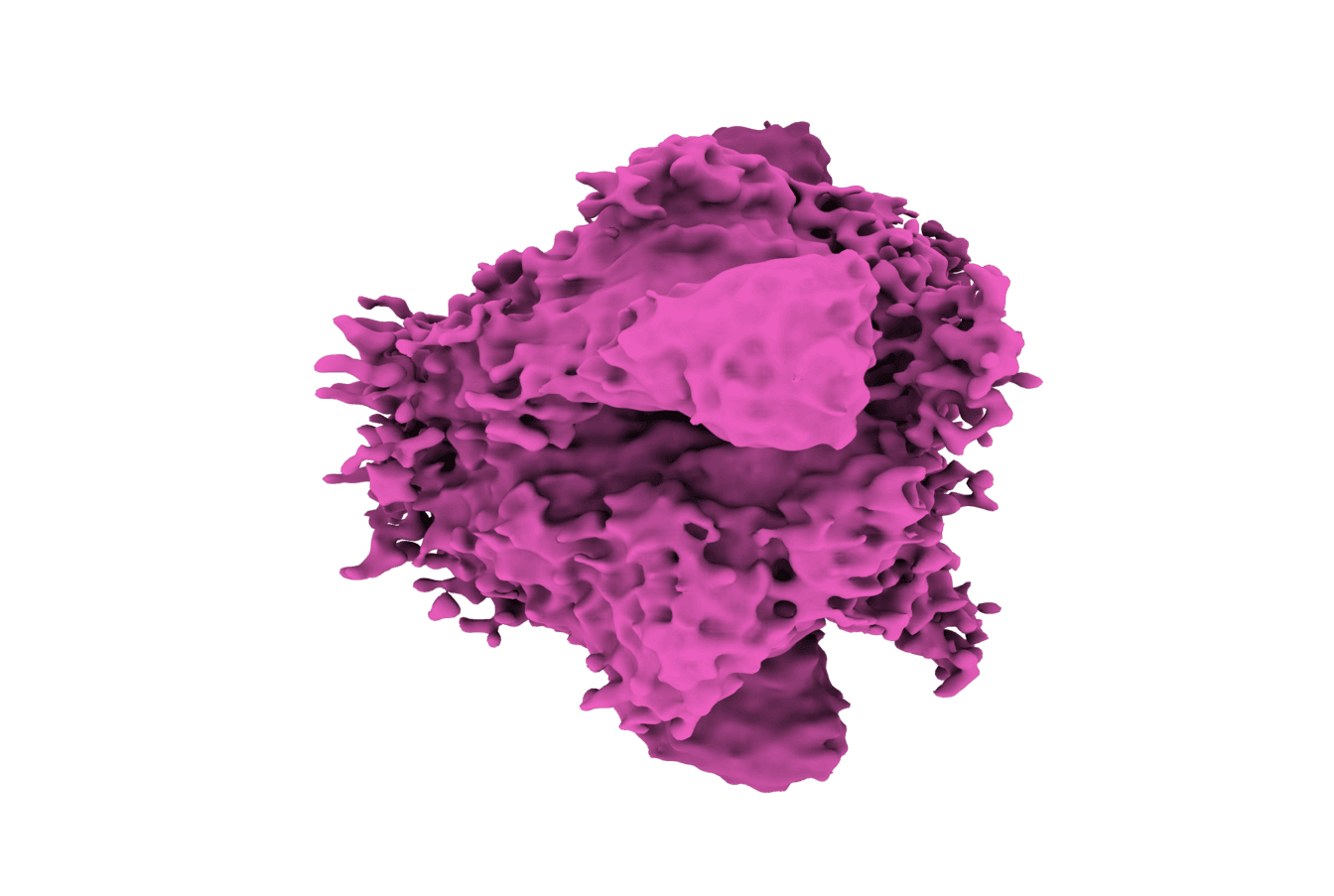}}
    \end{tabular}
    \setlength{\tabcolsep}{6pt} 
    \caption{Isosurfaces of 3D molecular densities from the Electron Microscopy Data Bank (EMDB). In our experiments, the molecules on the left are rotated around the Z axis, which corresponds to the depth direction here. The middle and right columns show the same molecules rotated by 90 degrees around the X and Y axis (respectively).}
    \label{fig:cryo-em-isosurfaces}
\end{figure}

\begin{table}
    \centering
    \small
    \setlength{\tabcolsep}{4pt}
    \caption{Relative computation time compared to exact OT solver of $16\times16\times16$ downscaled EMDB density maps. Results show mean and standard deviation across rotation angles. Lower is better.}
    \label{tbl:relative_time_results}
    \begin{tabular}{llcccc|cccc}
        \toprule
        & & \multicolumn{4}{c}{Upper Bounds} & \multicolumn{4}{c}{Lower Bounds} \\
        \cmidrule{3-6}\cmidrule{7-10}
        & & \multicolumn{1}{c}{\shortstack{Weighted-\\Cost}} & \multicolumn{1}{c}{\shortstack{Primal\\Upscaling}} & \multicolumn{2}{c}{\shortstack{Entropic\\Regularization}} & \multicolumn{1}{c}{\shortstack{Dual\\Upscaling}} & \multicolumn{1}{c}{\shortstack{Min-\\Cost}} & \multicolumn{2}{c}{\shortstack{Entropic\\Regularization}} \\
        Molecule & $p$ & $\kappa_2$ & $\kappa_2$ & $\varepsilon_1$ & $\varepsilon_4$ & $\kappa_2$ & $\kappa_2$ & $\varepsilon_1$ & $\varepsilon_4$ \\
        \midrule
        \multirow[t]{4}{*}{Ribosome}
        & \multirow[t]{2}{*}{1} & \textbf{3.7\%} & 175.6\% & 156.2\% & 155.8\% & 74.5\% & \textbf{3.5\%} & 154.1\% & 154.1\% \\
        & & {\scriptsize$\pm$3.2\%} & {\scriptsize$\pm$56.2\%} & {\scriptsize$\pm$111.3\%} & {\scriptsize$\pm$111.1\%} & {\scriptsize$\pm$21.8\%} & {\scriptsize$\pm$3.0\%} & {\scriptsize$\pm$109.8\%} & {\scriptsize$\pm$109.7\%} \\
        \cmidrule{2-10}
        & \multirow[t]{2}{*}{2} & \textbf{12.4\%} & 178.1\% & 569.8\% & 570.2\% & \textbf{12.2\%} & 12.5\% & 563.2\% & 563.6\% \\
        & & {\scriptsize$\pm$3.7\%} & {\scriptsize$\pm$226.5\%} & {\scriptsize$\pm$252.8\%} & {\scriptsize$\pm$252.3\%} & {\scriptsize$\pm$3.6\%} & {\scriptsize$\pm$5.0\%} & {\scriptsize$\pm$249.8\%} & {\scriptsize$\pm$249.6\%} \\
        \midrule
        \multirow[t]{4}{*}{SARS-CoV-2}
        & \multirow[t]{2}{*}{1} & \textbf{11.4\%} & 164.6\% & 69.1\% & 62.7\% & 69.9\% & \textbf{2.2\%} & 61.4\% & 61.2\% \\
        & & {\scriptsize$\pm$26.7\%} & {\scriptsize$\pm$96.4\%} & {\scriptsize$\pm$4.7\%} & {\scriptsize$\pm$15.1\%} & {\scriptsize$\pm$35.7\%} & {\scriptsize$\pm$0.3\%} & {\scriptsize$\pm$14.0\%} & {\scriptsize$\pm$14.8\%} \\
        \cmidrule{2-10}
        & \multirow[t]{2}{*}{2} & \textbf{19.7\%} & 350.5\% & 850.9\% & 658.8\% & 77.9\% & \textbf{19.0\%} & 642.5\% & 641.6\% \\
        & & {\scriptsize$\pm$3.5\%} & {\scriptsize$\pm$421.2\%} & {\scriptsize$\pm$895.4\%} & {\scriptsize$\pm$303.0\%} & {\scriptsize$\pm$184.2\%} & {\scriptsize$\pm$5.7\%} & {\scriptsize$\pm$295.5\%} & {\scriptsize$\pm$294.0\%} \\
        \midrule
        \multirow[t]{4}{*}{Yeast} 
        & \multirow[t]{2}{*}{1} & 4.8\% & 295.4\% & 20.8\% & \textbf{3.1\%} & 127.5\% & 3.3\% & 5.6\% & \textbf{0.9\%} \\
        & & {\scriptsize$\pm$4.9\%} & {\scriptsize$\pm$379.0\%} & {\scriptsize$\pm$31.2\%} & {\scriptsize$\pm$3.5\%} & {\scriptsize$\pm$153.4\%} & {\scriptsize$\pm$3.3\%} & {\scriptsize$\pm$4.6\%} & {\scriptsize$\pm$0.9\%} \\
        \cmidrule{2-10}
        & \multirow[t]{2}{*}{2} & \textbf{12.7\%} & 172.2\% & 919.4\% & 828.9\% & 13.3\% & \textbf{11.6\%} & 820.4\% & 819.7\% \\
        & & {\scriptsize$\pm$0.7\%} & {\scriptsize$\pm$69.3\%} & {\scriptsize$\pm$497.3\%} & {\scriptsize$\pm$215.6\%} & {\scriptsize$\pm$1.0\%} & {\scriptsize$\pm$0.7\%} & {\scriptsize$\pm$213.8\%} & {\scriptsize$\pm$212.5\%} \\\midrule
        \multirow[t]{4}{*}{HIV}
        & \multirow[t]{2}{*}{1} & \textbf{10.4\%} & 243.1\% & 163.1\% & 163.0\% & 99.6\% & \textbf{2.9\%} & 161.1\% & 161.1\% \\
        & & {\scriptsize$\pm$23.6\%} & {\scriptsize$\pm$215.3\%} & {\scriptsize$\pm$127.3\%} & {\scriptsize$\pm$127.1\%} & {\scriptsize$\pm$89.8\%} & {\scriptsize$\pm$2.7\%} & {\scriptsize$\pm$125.8\%} & {\scriptsize$\pm$125.8\%} \\
        \cmidrule{2-10}
        & \multirow[t]{2}{*}{2} & \textbf{40.7\%} & 198.2\% & 603.9\% & 603.9\% & 45.6\% & \textbf{40.4\%} & 597.0\% & 597.0\% \\
        & & {\scriptsize$\pm$97.7\%} & {\scriptsize$\pm$255.3\%} & {\scriptsize$\pm$252.2\%} & {\scriptsize$\pm$252.0\%} & {\scriptsize$\pm$114.3\%} & {\scriptsize$\pm$99.0\%} & {\scriptsize$\pm$249.3\%} & {\scriptsize$\pm$249.3\%} \\
        \bottomrule
    \end{tabular}
\end{table}
The quantization-based methods dominate in accuracy for the lower bounds of both $p\in\{1,2\}$, whereas for the upper bounds of the Wasserstein-1 metric, the upper bound based on entropic regularization with $\varepsilon=0.001N^p$ achieves the best accuracy, although at a significant computational cost.
The triangular symmetry of EMDB-14621 (SARS-CoV-2 spike protein) and EMDB-2484 (Trimeric HIV-1 envelope glycoprotein) seen in \cref{fig:emdb-rotations} are easily detectable as the dips at $120^\circ$ rotation angle.

\begin{figure}
    \centering
    \begin{tabular}{cc}
        \rotatebox{90}{\qquad\qquad \bf Ribosome} &
        \includegraphics[width=\textwidth]{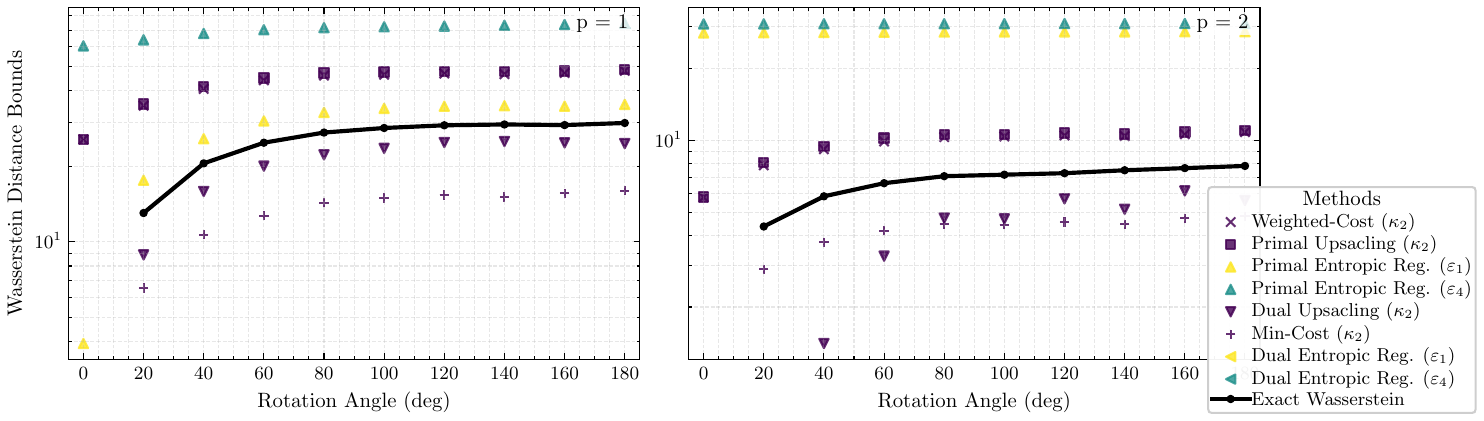}
        \\
        \rotatebox{90}{\qquad\qquad \bf SARS-CoV-2} &
        \includegraphics[width=\textwidth]{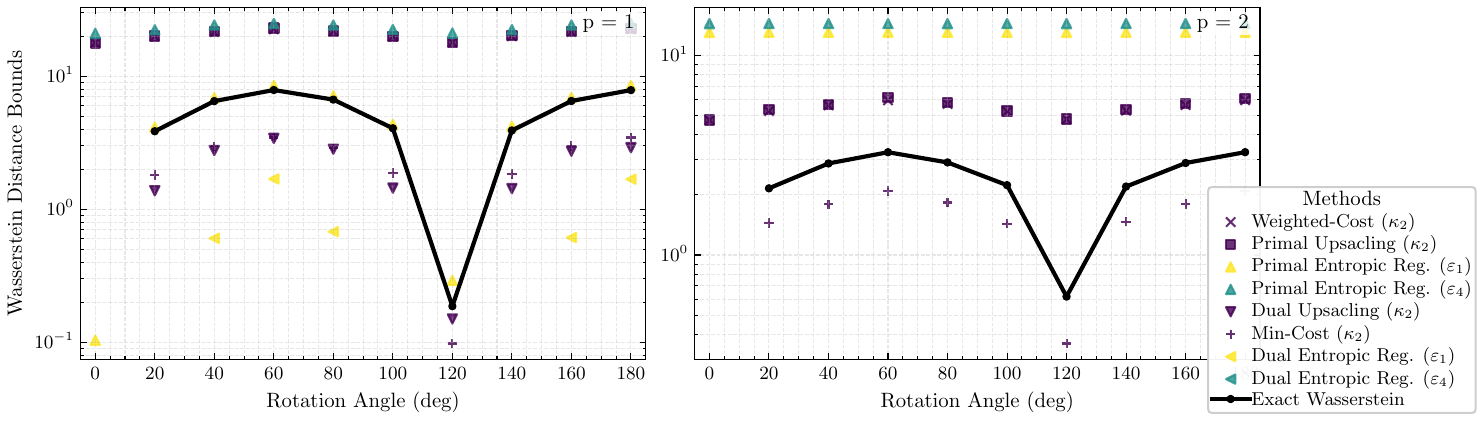}
        \\
        \rotatebox{90}{\qquad\qquad\quad \bf Yeast} &
        \includegraphics[width=\textwidth]{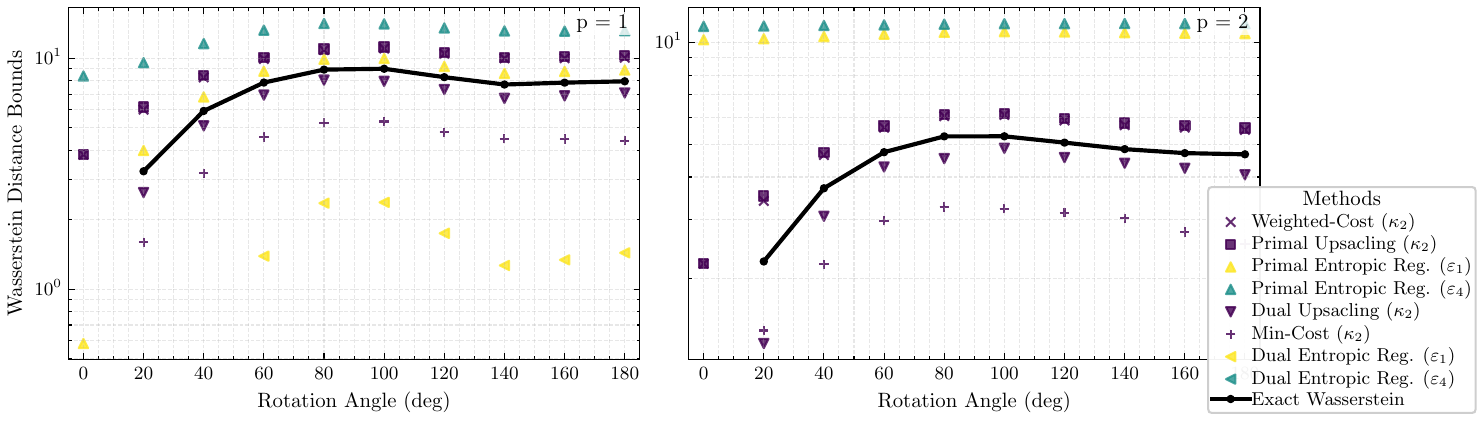}
        \\
        \rotatebox{90}{\qquad\qquad\quad \bf HIV} &
        \includegraphics[width=\textwidth]{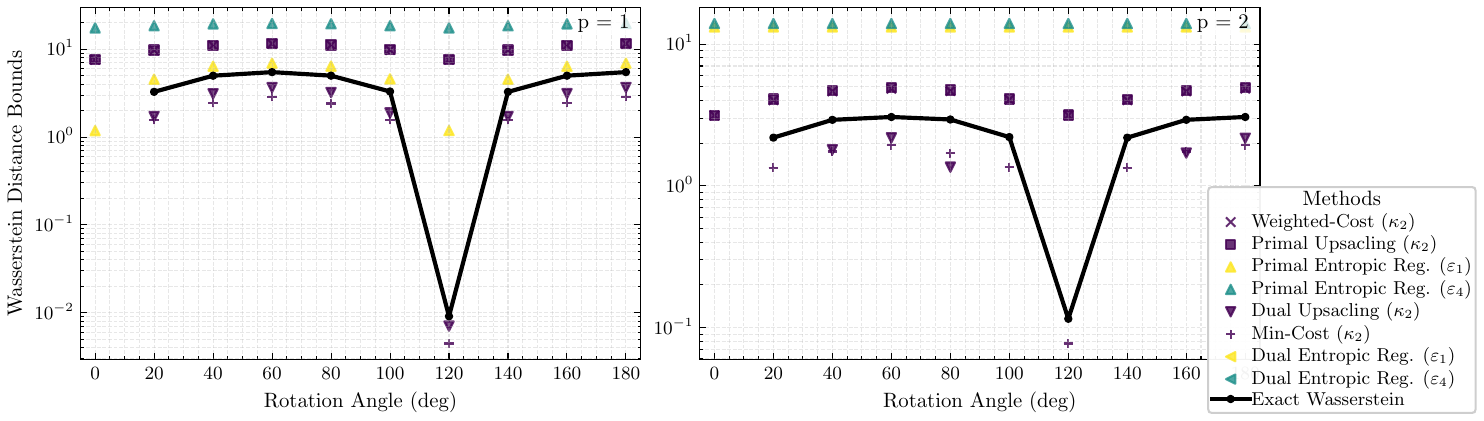}
        \end{tabular}
    \caption{Wasserstein-$p$ metric and bounds between rotated 3D density maps of the molecules described in \cref{tab:emdb}.
    From top to bottom: {\bf Ribosome}, {\bf SARS-CoV-2}, {\bf Yeast}, {\bf HIV}. The thick black line is the exact Wasserstein metric between the molecule and its rotated self, as a function of the rotation angle. The various upper and lower bounds are shown as different color markers. Note the drop around 120 degrees for the SARS-CoV-2 and HIV-1 spikes due to their 3-fold symmetry.}
    \label{fig:emdb-rotations}
\end{figure}

\FloatBarrier
\section{Proofs}\label{sc: proofs}
\subsection{Weighted total variation correction terms}\label{sc: proof of triangle inequality upper bound}
\begin{proof}[Proof of \cref{lm: triangle inequality upper bound}]
    Using the triangle inequality for the Wasserstein metric, we can write
    \begin{equation}
        \Wass_p(\mu,\nu) \le \Wass_p(\mu,\hat{\mu}) + \Wass_p(\hat{\mu},\hat{\nu}) + \Wass_p(\hat{\nu},\nu)
    \end{equation}
    controlling for each element separately, we have
    \begin{align}
           \Wass_p(\mu,\hat{\mu}) \le \mathcal{TV}_p^\vw(\hat{\mu},\mu) \qquad\text{and}\qquad
    \Wass_p(\hat{\nu},\nu) \le \mathcal{TV}_p^\vw(\nu,\hat{\nu})
    \end{align}
 
    by the property of weighted total variation, and
    \begin{align}
        \Wass_p(\hat{\mu},\hat{\nu})
        =
        \left(
        \min_{\pi \in \Pi(\hat{\mu},\hat{\nu})}
        \inner{\pi,C}
        \right)^{\frac{1}{p}}
        \le
        \inner{\hat{\pi},C}^{\frac{1}{p}}
    \end{align}
    by evaluating the transport cost using a coupling in the problem's original space.
\end{proof}
\begin{proof}[Proof of \cref{th: $L^1$ upper bound on weighted total variation}]
    Consider the definition of weighted total variation,
    \begin{align}
        \mathcal{TV}_p & (\hat{\mu}, \mu; w) + \mathcal{TV}_p(\nu, \hat{\nu}; w)
        \\  \nonumber &= 2^{1-\frac{1}{p}} \inner{w,|\hat\mu-\mu|}^{\frac{1}{p}} + 2^{1-\frac{1}{p}} \inner{w,|\nu - \hat\nu|}^{\frac{1}{p}}
        \\  \nonumber &=
        2^{1-\frac{1}{p}} \left(\left(\sum \rho(\bar{x},x)^p |\hat{\mu}_x - \mu_x| \right)^\frac{1}{p} + \left(\sum \rho(\bar{x},x)^p |\nu_x - \hat{\nu}_x| \right)^\frac{1}{p} \right)
        \\  \nonumber &\le
        2^{1-\frac{1}{p}} \Big(r\|\hat{\mu} - \mu\|_1^\frac{1}{p} + r\|\nu - \hat{\nu}\|_1^\frac{1}{p} \Big)
                       &                                                         & \text{bounding radius} \nonumber
        \\ &\le
        2^{2-\frac{2}{p}} \left(\|\hat{\mu} - \mu\|_1 + \|\nu - \hat{\nu}\|_1  \right)^\frac{1}{p} r
                       &                                                         & \text{Jensen's inequality} \nonumber
        \\ &<
        2^{2-\frac{2}{p}} \xi^\frac{1}{p} r
                       &                                                         & \text{convergence criteria} \nonumber
    \end{align}
\end{proof}

\subsection{Proof of \cref{th: weighted cost UB} \label{sc: proof of weighted cost UB}}
First, let us consider the following lemma discussing a coupling constructed ad hoc using coarsened measures.
\begin{lemma}\label{lm: admissible quantization coupling}
    Let  $\mu,\nu$ measures with set of admissible couplings $\Pi(\vmu,\vnu)$, the trivial coupling $\pi_\otimes$ , and  $\tilde{\mu}$ and $\tilde{\nu}$ the respective coarsened measures. For any coupling $\tilde{\pi}\in\Pi(\tilde{\mu},\tilde{\nu})$, the measure on the product space $\calX\times\calY$

    \begin{equation}\label{eq: quantization coupling}
        \pi_{\tilde{\pi}}(x,y) := \frac{\tilde{\pi}_{k_\sX(x)\ell_\sY(y)}}{\mu(X_{k_\sX(x)})\nu(Y_{\ell_\sY(y)})}\pi_{\otimes} (x,y)
    \end{equation}
    is an admissible coupling $\pi_{\tilde{\pi}} \in \Pi(\vmu,\vnu)$, where the coarsening inverse index functions $k_\sX(x),\,\ell_\sY(y)$ are defined as  $k_\sX(x):=\{k: x\in X_k\}, \ell_\sY(y):=\{\ell: y\in Y_\ell\}$.
\end{lemma}

\begin{proof}[Proof of \cref{lm: admissible quantization coupling}]
    Following Definition 1.1 (Coupling) from \citep{villaniOptimalTransportOld2009}, one can show $\pi_{\tilde{\pi}}(x, y)$ \cref{eq: quantization coupling} is admissible. For $\varphi, \psi$ be any integrable measurable functions on $\calX,\calY$ respectively, than $\pi_{\tilde{\pi}}(x, y)$ admits

    \begin{align}
         & \int_{\calX \times \calY} \big(\varphi(x)+\psi(y)\big)\mathrm{d}\pi_{\tilde{\pi}}(x, y) \\  \nonumber &= 
        \sum_{k,\ell} \int_{X_{k} \times Y_{\ell}} \big(\varphi(x)+\psi(y)\big) \mathrm{d}\pi_{\tilde{\pi}}(x, y)\\  \nonumber
        &=
        \sum_{k,\ell} \int_{X_{k}\times Y_{\ell}}\big(\varphi(x)+\psi(y)\big) \frac{\bm{\Pi}_{k\ell}}{\mu(X_{k})\nu(Y_{\ell})} \mathrm{d}\mu(x) \mathrm{d}\nu(y) \qquad\quad\text{plug-in coupling's definition} \\ \nonumber
        &=
        \sum_{k,\ell} \Big(\frac{\tilde{\pi}_{k\ell}}{\mu(X_{k})\nu(Y_{\ell})}\int_{X_{k}\times Y_{\ell}}\varphi(x) \mathrm{d}\mu(x) \mathrm{d}\nu(y) + \frac{\tilde{\pi}_{k\ell}}{\mu(X_{k})\nu(Y_{\ell})} \int_{X_{k}\times Y_{\ell}}\psi(y)  \mathrm{d}\mu(x) \mathrm{d}\nu(y) \Big) \\ \nonumber
        &=
        \sum_{k,\ell} \frac{\tilde{\pi}_{k\ell}}{\mu(X_{k})}\int_{X_{k}}\varphi(x) \mathrm{d}\mu(x) + \sum_{k,\ell} \frac{\tilde{\pi}_{k\ell}}{\nu(Y_{\ell})} \int_{Y_{\ell}}\psi(y) \mathrm{d}\nu(y) \qquad\text{sum over marginals}\\ \nonumber
        &=
        \sum_{k} \int_{X_{k}}\varphi(x) \mathrm{d}\mu(x) + \sum_{\ell} \int_{Y_{\ell}}\psi(y) \mathrm{d}\nu(y)\\ \nonumber
        &=
        \int_{\calX}\varphi(x) \mathrm{d}\mu(x) + \int_{\calY}\psi(y) \mathrm{d}\nu(y)
    \end{align}
\end{proof}
Next, we consider the transport cost of such a coupling.
\begin{lemma}\label{lm: weighted cost transport identity}
    The transport loss assigned by the cost $c(x,y)$ and a coupling $\pi_{\tilde{\pi}}$ identifies with the coarse transport loss assigned by marginally weighted cost $\bar{C}$ \cref{eq: marginally weighted cost} and the coarse coupling $\tilde{\pi}$,
    \begin{equation}
        \inner{\pi_{\tilde{\pi}}, C} = \inner{\tilde{\pi}, \bar{C}}
    \end{equation}
\end{lemma}
\begin{proof}
    \begin{align}
        \inner{\pi_{\tilde{\pi}}, C} & =
        \sum_{i,j} c(x_i,y_j) \pi_{\tilde{\pi}}(x_i,y_j) =
        \sum_{k,\ell} \sum_{\substack{x\in X_k                                   \\y\in Y_\ell}} c(x,y)\pi_{\tilde{\pi}}(x,y) \\  \nonumber &=
        \sum_{k,\ell} \sum_{\substack{x\in X_k                                   \\y\in Y_\ell}} c(x,y) \frac{\tilde\pi_{k\ell}}{\mu(X_{k})\nu(Y_{\ell})} \pi_{\otimes}(x,y) \\  \nonumber &=
        \sum_{k,\ell} \frac{1}{\mu(X_{k})\nu(Y_{\ell})} \sum_{\substack{x\in X_k \\y\in Y_\ell}} c(x,y) \mu(x)\nu(y)\,\tilde\pi_{k\ell} \\  \nonumber &=
        \sum_{k,\ell} \bar{C}_{k\ell} \tilde\pi_{k\ell} =
        \inner{\tilde{\pi}, \bar{C}}
    \end{align}
\end{proof}

Finally, we can write
\begin{proof}[Proof of \cref{th: weighted cost UB}]
    Based on admissibility of $\pi_{\tilde{\pi}}$ shown in \cref{lm: admissible quantization coupling} the transport cost
    \begin{align}
        \inner{\pi_{\tilde{\pi}}, C} \ge L_C(\mu,\nu),\ \forall \tilde{\pi}\in\Pi(\tilde{\mu},\tilde{\vnu}).
    \end{align}
    In particular, for $\tilde{\pi}^* = \argmin_{\tilde{\pi} \in \Pi(\tilde{\vmu}, \tilde{\vnu})} \inner{\tilde{\pi},\bar{C}}$,
    \begin{equation}
        \inner{\tilde{\pi}^*,\bar{C}} = \inner{\pi_{\tilde{\pi}^*},C} \ge L_C(\mu,\nu)
    \end{equation}
    by the identity shown in \cref{lm: weighted cost transport identity}.
\end{proof}

\subsection{Additional proofs}\label{sec: additional proofs}
\begin{proof}[Proof of \cref{th: locally-minminal cost LB}]
    Consider
    \begin{align}
        \pi^* = \argmin_{\pi\in\Pi(\vmu,\vnu)}\inner{\pi,C}
    \end{align}
    and coarsening of the optimal coupling
    \begin{align}
        \hat{\pi}^*_{k\ell} := \sum_{\substack{
                x\in X_k\\
                y\in Y_\ell}}
        \pi^*(x,y)
    \end{align}
    such that,
    \begin{align}
        L_C(\vmu,\vnu) 
        &= \inner{\pi^*,C} \\ \nonumber
        &= \sum_{\substack{
        x\in \calX                           \\
                y\in \calY}}
        \rho(x,y)^p \pi^*(x,y) \\ \nonumber
        &= \sum_{k,\ell}\sum_{\substack{
        x\in X_k                             \\
                y\in Y_\ell}}
        \rho(x,y)^p \pi^*(x,y)               \\ \nonumber
        &\ge  \sum_{k,\ell} C_{k\ell}^{\min}
        \sum_{\substack{
        x\in X_k                             \\
                y\in Y_\ell}}
        \pi^*(x,y) \\ \nonumber
        &= \inner{\hat{\pi}^*, C^{\min}} \\ \nonumber
        &\ge \min_{\tilde{\pi}\in\Pi(\tilde{\vmu},\tilde{\vnu})}\inner{\tilde\pi, C^{\min}} \\ \nonumber
        &= L_{C^{\min}}(\tilde{\mu},\tilde{\nu}).
    \end{align}
\end{proof}

\begin{proof}[Proof of \cref{lm: normalized coupling}]
    Let $\tilde{\pi}^* \in \Pi(\tilde{\vmu}, \tilde{\vnu})$ be the coarse optimal coupling and $\tK$ be the positive valued normalized kernel tensor satisfying $\sum_{t \in [\kappa]^{2d}} \tK_t = 1$. Recall that $\hat{\tP} = \tilde{\tP}^* \otimes \tK$ and $\hat{\pi}$ is obtained by reshaping $\hat{\tP}$.
    
    First, we show that the sum of all elements equals 1:
    \begin{align}
        \sum_{i=1}^{N^d}\sum_{j=1}^{N^d} \hat{\pi}_{ij}
        &= \sum_{t \in [\kappa]^{2d}} \sum_{k=1}^{n^d}\sum_{\ell=1}^{n^d} \tilde{\pi}^*_{k\ell} \tK_t \\ \nonumber
        &= \sum_{k=1}^{n^d}\sum_{\ell=1}^{n^d} \tilde{\pi}^*_{k\ell} \sum_{t \in [\kappa]^{2d}} \tK_t \\ \nonumber
        &= \sum_{k=1}^{n^d}\sum_{\ell=1}^{n^d} \tilde{\pi}^*_{k\ell} \cdot 1 
        = 1
    \end{align}
    where the last equality follows from $\tilde{\pi}^*$ being a coupling.
    
    Second, we show that $\hat{\pi}$ is non-negative. Since $\tilde{\pi}^*$ is a coupling, it is non-negative, and $\tK$ is a positive-valued kernel, their tensor product $\hat{\tP}$ and its reshaped form $\hat{\pi}$ are non-negative.
    
    Thus, $\hat{\pi}$ satisfies all the properties of a coupling measure.
\end{proof}

\newpage
\section{Table of notations}
\begin{table*}[ht] \small
    \caption{Table of Notations}
    \begin{center}
        \begin{tabular}{@{}llp{9cm}@{}}
            \toprule
            \textbf{Notation}                    & \textbf{Category} & \textbf{Description}                                                 \\
            \midrule
            \multicolumn{3}{l}{\textit{Sets and Spaces}}                                                                                    \\
            \addlinespace
            $\R_+$                               & Set               & Non-negative real numbers                                            \\
            $[n]$                                & Set               & Set of integers $\{1,\dots,n\}$                                      \\
            $\Sigma_N$                           & Space             & Probability simplex $\{(p_1,\dots,p_N) \in \R_+^N: \sum_i p_i = 1\}$ \\
            $\calX, \calY$                       & Set               & Point sets where measures are defined                                \\
            $\sX$                                & Set               & Set of non-overlapping hypercubes covering the grid                  \\
            $\tilde{\calX}, \tilde{\calY}$       & Set               & Coarse grids (set of hypercube centers)                              \\
            $X_k, Y_\ell$                        & Set               & Individual hypercubes in the partition                               \\
            \midrule
            \multicolumn{3}{l}{\textit{Measures and Vectors}}                                                                               \\
            \addlinespace
            $\zeros_n, \ones_n$                  & Vector            & All-zeros and all-ones vectors in $\R^n$, respectively \\
            $\mu, \nu$                           & Measure           & Discrete probability measures                                        \\
            $\vmu, \vnu$                         & Vector            & Vector representations of measures $\mu, \nu$                        \\
            $\tilde{\mu}, \tilde{\nu}$           & Measure           & Coarsened measures                                                   \\
            $\tilde{\vmu}, \tilde{\vnu}$         & Vector            & Vector representations of coarsened measures                         \\
            $\vf, \vg$                           & Vector            & Kantorovich potentials                                               \\
            $\va, \vb$                           & Vector            & Sinkhorn scaling vectors                                             \\
            \midrule
            \multicolumn{3}{l}{\textit{Matrices and Tensors}}                                                                               \\
            \addlinespace
            $C$                                  & Matrix            & Ground-cost matrix                                                   \\
            $\tilde{C}$                          & Matrix            & Coarse cost matrix (center-based)                                    \\
            $\bar{C}$                            & Matrix            & Coarse cost matrix (average-based)                                   \\
            $\pi$                                & Matrix            & Transport plan (coupling matrix)                                     \\
            $\pi^*$                              & Matrix            & Optimal transport coupling                                           \\
            $\tilde{\pi}$                        & Matrix            & Coarse coupling                                                      \\
            $\tK$                                & Tensor            & Normalized kernel tensor                                             \\
            \midrule
            \multicolumn{3}{l}{\textit{Functions and Operations}}                                                                           \\
            \addlinespace
            $\rho$                               & Function          & Distance metric                                                      \\
            $\Wass_p$                            & Function          & Wasserstein-$p$ metric                                               \\
            $L_C$                                & Function          & Optimal transport cost for ground-cost $C$                           \\
            $\mathcal{TV}_p^\vw$                 & Function          & Weighted total variation                                             \\
            $\otimes$                            & Operation         & Tensor product                                                       \\
            $\odot, \oslash$                     & Operation         & Pointwise multiplication, division                                   \\
            $\inner{\cdot,\cdot}$                & Operation         & Standard vector/matrix inner product                                 \\
            \midrule
            \multicolumn{3}{l}{\textit{Parameters and Constants}}                                                                           \\
            \addlinespace
            $p$                                  & Parameter         & Order of Wasserstein metric ($p \geq 1$)                             \\
            $\kappa$                             & Parameter         & Scale factor for coarsening                                          \\
            $\xi$                                & Parameter         & Convergence threshold for fitting                                    \\
            $N$                                  & Constant          & Side length of regular grid                                          \\
            $d$                                  & Constant          & Dimension of the space                                               \\
            $n$                                  & Constant          & Side length of coarse grid ($n = N/\kappa$)                          \\
            $\bar{x}$                            & Constant          & Center point of space $\calX$                                        \\
            $r$                                  & Constant          & Radius of space $\calX$                                              \\
            $\Delta_{\hat\mu}, \Delta_{\hat\nu}$ & Constant          & Marginal weighed total variation corrections                         \\
        \midrule
            \multicolumn{3}{l}{\textit{Code}}                                                                                               \\
            \addlinespace
            \code{AvgPool}, \code{SumPool}                       & Function          & Average/sum pooling layer with identical kernel size and stride         \\
            \code{mean}                          & Function          & Mean of a set of points                                              \\
            \code{reshape}                       & Function          & Cardinality preserving tensor shape transformation                   \\
            \bottomrule
        \end{tabular}
    \end{center}
    \label{tbl: notations}
\end{table*}

\end{document}